\documentclass[preprint,5p,times,numbered]{elsarticle}
\usepackage{booktabs} 
\usepackage{natbib}
\usepackage{amssymb,amsfonts}
\usepackage{algorithmic}
\usepackage{graphicx}
\usepackage{textcomp}
\usepackage{xcolor}
\usepackage{subfigure}
\usepackage{amsmath, pgfplots, graphicx, siunitx, tikzscale}
\usepackage{pgfplots,gincltex}
\usepackage{siunitx}
\usepackage{pgfplots}
\usepackage{mathtools}
\usepackage{amsthm}
\usepackage[ruled,vlined,linesnumbered]{algorithm2e}
\usepackage{graphicx}
\usepackage{subcaption}
\usepackage{multirow}
\usepackage{makecell}
\usepackage{appendix}

\DeclareMathOperator*{\argmax}{arg\,max}

\newtheorem{theorem}{Theorem}

\newtheorem{prop}{Proposition}

\usepackage[font=normalsize]{caption}
\usepackage{url}

\begin{document}
	\begin{sloppypar}
		\begin{frontmatter}
			
			\title{Adaptive NAD: Online and Self-adaptive Unsupervised \emph{N}etwork \emph{A}nomaly \emph{D}etector}
			
			\author[1]{Yachao~Yuan} 
			\affiliation[1]{organization={School of Future Science and Engineering, Soochow University},
				city={Suzhou},
				state={Jiangsu},
				country={China}}
			
			
			\affiliation[3]{organization={School of Cyber Science and Engineering, Southeast  University},
				city={Nanjing},
				state={Jiangsu},
				country={China}}
			
			
			\author[3]{Yu~Huang} 
			
			\author[1]{Yingwen~Wu*} 
			\ead{Corresponding author: (ywwu@suda.edu.cn)}

			\begin{abstract}
				The widespread usage of the Internet of Things (IoT) has raised the risks of cyber threats; thus, developing Anomaly Detection Systems (ADSs) that can adapt to evolving traffic pattern is critical. Previous studies primarily focused on offline unsupervised learning methods to safeguard ADSs, which is not applicable in practical real-world applications. In this paper, we design Adaptive NAD, an online and self-\emph{Adaptive} unsupervised \emph{N}etwork \emph{A}nomaly \emph{D}etection framework for security domains. A two-layer anomaly detection strategy is proposed to generate reliable high-confidence pseudo-labels. Then, an online training scheme is introduced to update Adaptive NAD by a novel threshold calculation technique. Experimental results demonstrate that Adaptive NAD achieves the lowest false alarm  rate (1.33\%, 0.71\%, and 0.08\%) and has a more than 3 times faster online inference latency compared with state-of-the-art solutions on the CIC-Darknet2020, NSL-KDD, and Edge-IIoTset datasets, respectively. The code is released at https://github.com/MyLearnCodeSpace/Adaptive-NAD.
			\end{abstract}
			
			
			
			
		\end{frontmatter}

		\section{Introduction} \label{sec:introduction}
		The rapid expansion of the Internet of Things (IoT) has revolutionized many areas, such as healthcare, manufacturing, and autonomous driving \cite{esteva2019guide,wang2018deep,levinson2011towards}. However, it also increases risks of cybersecurity attacks, for example, UnitedHealth’s 872 Million Cyberattack in 2024~\footnote{\url{https://www.sec.gov/Archives/edgar/data/731766/000073176624000146/a2024q1exhibit991.htm}} and the I.M.F. breach in 2024~\footnote{\url{https://www.bleepingcomputer.com/news/security/international-monetary-fund-email-accounts-hacked-in-cyberattack/}}. Anomaly Detection Systems (ADSs) are essential to discover potential network security problems. The advances in machine learning, particularly deep learning, have significantly accelerated this progress \cite{bovenzi2023network,saba2022anomaly,chen2024network,chen2019unsupervised}. 
		
		In the real world, the environment is dynamic and ever-changing. Both normal data and anomalies are evolving over time. Thus, traditional offline training-based methods are not applicable in practical applications. Given that machine learning models strongly rely on learned data patterns, any shifts in normal/abnormal behavior could damage anomaly detection performance. Hence, ADSs demand constant adjustment to the model's changing behavior over time in an online manner. Online learning provides a potential way to improve models' adaptability by continuously retraining the model. Moreover, unsupervised learning-based models can directly learn from only normal data without any known threats \cite{tuli2023deepft,liu2024timesurl}; they are more favorable for online learning tasks of anomaly detection due to the inherent lack of labeled anomalies in historical data and the unpredictable and highly varied nature of anomalies. 
		
		In practical security-related applications, it is challenging to establish an effective online unsupervised learning scheme demanding expeditious adaptation of the model to a continuous incoming of unlabeled data.
		
		Most recent existing research \cite{audibert2020usad,zhang2023real,han2021deepaid} conducts anomaly detection by using a fixed threshold, which tends to fail to detect anomalies when concept drifts exist. To address this, a dynamic threshold is utilized to distinguish normal from anomalies in \cite{tuli2023deepft,odiathevar2021online}, but suffers from high false positive rates. Additionally, most of the existing literature \cite{audibert2020usad,li2021multivariate,deng2021graph,han2021deepaid,zong2018deep,han2022learning} trains their models offline without any adaption to drifting patterns, making them not applicable in real-world applications. 
		There is only limited research on online unsupervised anomaly detection, and most of them, like \cite{baldini2022online,wang2022distributed,chen2021daemon}, consider \textquotedblleft online\textquotedblright~as the model being trained offline and utilized for real-time inferences without model updating. 
		

		\textbf{Our work.} To address the above issues, we propose a general framework, named \textbf{Adaptive NAD}, to improve the unsupervised online training of deep learning-based anomaly detection in security applications. The high-level goal of Adaptive NAD is to design a fast and accurate anomaly detection model that can be trained in an online and unsupervised manner and meets several practical security requirements (such as generalization, low false positives and low latency). To this end, we design two key strategies in Adaptive NAD: (i) a two-layer unsupervised anomaly detection scheme by leveraging the advantages of both supervised and unsupervised models to generate highly confident pseudo labels in highly imbalanced network streaming data. (ii) Online training scheme for both fine-tuning models and regenerating loss distributions aligned with the evolving network events, and making real-time predictions of network anomalies given streaming network traffic. 
		
		Overall, Adaptive NAD is a novel online and self-adaptive unsupervised network anomaly detection framework for its real-time predictions, that can be utilized with any pair of unsupervised-supervised models for online unsupervised anomaly detection tasks.
		
		The contributions of this paper are mainly threefold:
		\begin{itemize}
			\item A novel online and self-adaptive unsupervised framework, named Adaptive NAD, that can be utilized with any pair of unsupervised and supervised learning models for unsupervised anomaly detection.
			\item A new dynamic threshold technique with theoretical proof for determining high-confidence pseudo labels can be employed for fine-tuning online models.
			\item We compare it with the state-of-the-art models on different datasets and the proposed Adaptive NAD achieves the lowest false alarm rate (1.33\%, 0.71\%, and 0.08\%) and has more than 3 times faster online inference latency. 
		\end{itemize}
		
		The rest of the paper is outlined as follows: we first give an overview of the system design in Section \ref{sec:systemDesign}, including design goals, system components, and details of the models used in Adaptive NAD. Then, the proposed novel techniques, i.e., two-layer anomaly detection strategy (ITAD-S), online training scheme, and threshold calculation technique, are presented in Section \ref{sec:Methodologies} along with a comprehensive complexity and storage analysis of the proposed Adaptive NAD framework. Following it, Section \ref{sec:Evaluation} describes the experimental setup and results. Existing related literature is discussed in Section \ref{sec:RelatedWork}. Finally, we conclude the paper in Section \ref{sec:Conclusions}.
		
		\begin{figure*}[t]
			\centering
			\includegraphics[width=0.9\linewidth]{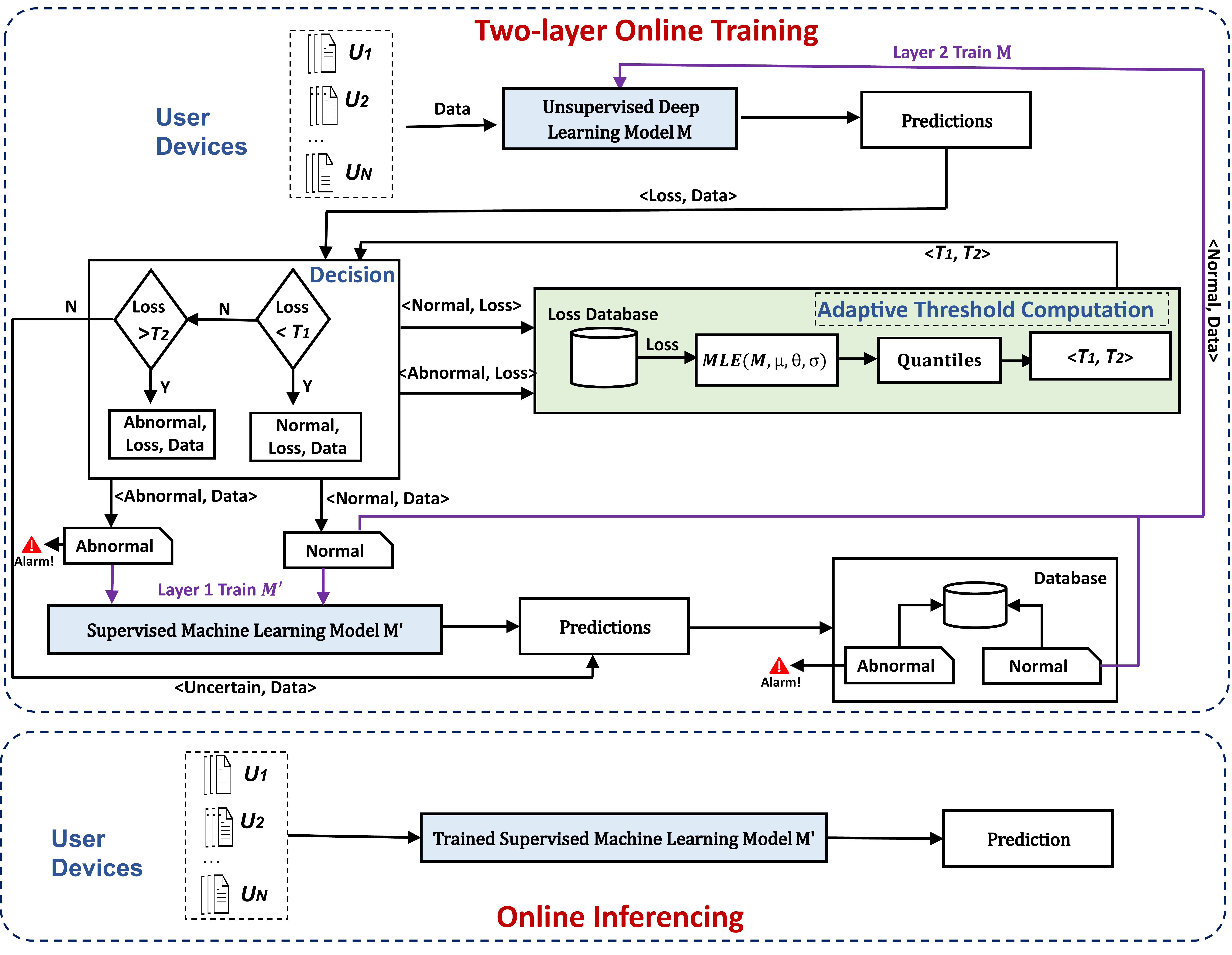}
			\caption{The proposed Adaptive NAD framework.}
			\label{fi:framework}
		\end{figure*}

		\section{System Design} \label{sec:systemDesign}
		This section presents the design goals, system components of the proposed Adaptive NAD, and details of the two key models utilized in Adaptive NAD.
		\subsection{Design Goals} \label{subsection:Design Goals}
		The design goals of our proposed approach are explained below:
		\begin{enumerate}
			\item \textbf{Capable of handling diverse data:}
			\textcolor{black}{Network traffic} collected from different devices varies greatly regarding their format and content. A generalized ADS needs to be capable of analyzing various network formats.
			\item \textbf{High accuracy but low false positive rate:}
			Anomaly detection should correctly predict both normal and abnormal activities; in other words, achieving a high accuracy rate but a low false positive rate.
			\item \textbf{Low computational cost:} High computational cost incurred during anomaly detection leads to increased latency, deployment, and maintenance costs.
		\end{enumerate}
		
		\subsection{System Components} \label{subsection:architecutureCom}
		As shown in Fig. \ref{fi:framework}, the proposed system comprises the following components:
		\begin{itemize}
			\item \underline{User Devices:} \textcolor{black}{User devices are terminal devices in distributed IoT systems that serve as the primary source of network traffic data. These devices generate, collect, and transmit data to the network, enabling real-time network anomaly detection.}
			\item \underline{Unsupervised deep Learning Model $M$:} This model is used to learn the network traffic using unsupervised deep learning models. In this paper, an unsupervised learning model (i.e., NAD) is utilized to learn the pattern of the normal network traffic (see Section~\ref{subsection:LSTM}). 
			\item \underline{Predictor:} After the deep learning model $M$ is trained on the normal network traffic, $M$ can be used to predict the new upcoming network data.
			\item \underline{Adaptive Threshold Computation:} This model generates two thresholds $T_1$ and $T_2$ by using Quantiles and Maximum Likelihood Estimation (MLE) (see Section \ref{subsec:ThresholdCalculation}). Thresholds are adaptively changing to incorporate the characteristics of new network traffic to identify normal and abnormal network data.
			\item \underline{Decision:} This module is employed to select network data with the (1-$\alpha$) confidence level by dynamic thresholds, for the training of the machine learning model $M^{'}$ based on the predicted network loss data from the unsupervised deep learning model $M$. In addition, the loss of accurately selected network events is utilized to update the thresholds.
			\item \underline{Supervised Machine Learning Model:} This module uses simple machine learning methods, such as a Random Forest classifier (RF), to predict the uncertainty network traffic from an unsupervised deep learning model. This improves the detection performance of unsupervised learning models significantly. 
			\item \underline{Database:} The stored dataset is used to update the machine learning model.
		\end{itemize}

		\subsection{The Network Anomaly Detection Model (NAD)} \label{subsection:LSTM}
		The Network Anomaly Detection Model (NAD) is utilized as the unsupervised deep learning model $M$, which is adapted from VAE~\cite{kingma2013auto}.
		We use the Long Short-Term Memory (LSTM) for the encoder and decoder phase, and use a Gaussian distribution for $p(z)$ in VAE. NAD is a generative model whose aim is to perform and learn a probabilistic model $P$ from a sample $X$ with an unknown distribution. 
		It overcomes the three following drawbacks from its predecessors: Firstly, 
		it makes strong assumptions about the structure of the data. Secondly, 
		due to the approximation methods, they end up with sub-optimal models. Lastly, they generally use significant computational resources for the inference procedure like Markov Chain Monte Carlo (MCMC)~\cite{doersch2016tutorial}.
		
		The idea behind this model is to use latent variables $Z$, which might be understood as hidden variables that influence the behavior of $X$; thus, having the initial sample $X$, we need to train a model to learn the distribution of $Z$ given the distribution $X$. This allows us to eventually make the inference over $X$. We now express the model in probability terms, as follows. Specifically, for $x\in X$, $z\in Z$:
		\begin{equation} \label{eq:01}
			p(x)=\int_{z} p(x,z)dz
			=\int_{z} p(x|z)p(z)dz.
		\end{equation}
		
			Since the integral with respect to $z$ in the Eq.~\eqref{eq:01} is difficult to compute due to the high dimensional space of $Z$, we can express the log maximum likelihood of $logp_\theta(x)$ as \cite{kingma2013auto}:
			\begin{align}
				logp_\theta(x)=D_{kl}[p_\theta(z|x)||p_\theta(z)]+\mathcal{L}(x,\phi,\theta),
			\end{align}
			where in this equation, the first term $D_{kl}$ is the Kullback-Leibler divergence (a non negative term) and the second term can be seen as a lower bound of likelihood $logp_\theta(x)$. Thus, this bound is referred to as the variational lower bound, which can be expressed as: \cite{kingma2013auto}:
			\begin{align}
				logp_\theta(x) \geq \mathcal{L}(x,\phi,\theta)= E[logp_\theta(x,z)-logp(z|x)].
			\end{align}
			Due to the intractability of $p(z|x)$, NAD uses a learning model $q_{\phi}(z|x)$ that approximates $p(z|x)$ with $\phi$ as the parameter of distribution $q$. Then, we can rewrite the equation as:
			\begin{align}
				&\mathcal{L}(x,\phi,\theta)= E_{q_{\phi}(z|x)}[logp_\theta(x,z)-logq_\phi(z|x)] \nonumber \\
				&=-D\_kl[q_\phi(z|x)||p_\theta(z)]+E_{q_{\phi}(z|x)}[logp_\theta(x|z)], \label{eq:L}
			\end{align}
			where $p(z)$ is in general a simple distribution like Gaussian for continuous variable or Bernoulli for binary variable. Finally, we want to maximize $logp_\theta(x)$, but in Eq.~\eqref{eq:L} we observe that we need to minimize the Kullback-Leibler divergence. Thus, NAD optimizes the function $\mathcal{L}$ for this purpose~\cite{kingma2013auto}.

		Fig.~\ref{fi:models}(a) shows the NAD diagram. In its Encoder phase, we can see that given an input $X$, the model uses the LSTM encoder $q_{\phi}(z|x)$ with the multilayer network $f(x,\phi)$ to approximate the probability $p(z/x)$, which is the distribution of the latent variable $Z$ given $X$. In its Decoder phase, with the knowledge of $Z$, the model uses a LSTM decoder $p_\theta(x|z)$ with the multilayer network $g(z,\theta)$ to approximate the probability $p(x/z)$, which is the distribution of the input variable $X$ given its latent variable $Z$. Finally, the NAD model solves the optimization problem in Eq. (3) using the Encoder and Decoder learning models given as Eq. \eqref{eq:encoder} and Eq. \eqref{eq:decoder}, respectively.
		\begin{equation}
			\label{eq:encoder}
			q_\phi(z|x)=q(z,f(x,\phi)), 
		\end{equation}
		\begin{equation}
			\label{eq:decoder}
			p_\theta(x|z)=p(x,g(x,\theta)). 
		\end{equation}
		\begin{figure*}[t!]
			\centering
			\begin{minipage}{.5\textwidth}
				\centering
				\includegraphics[width=0.9\linewidth]{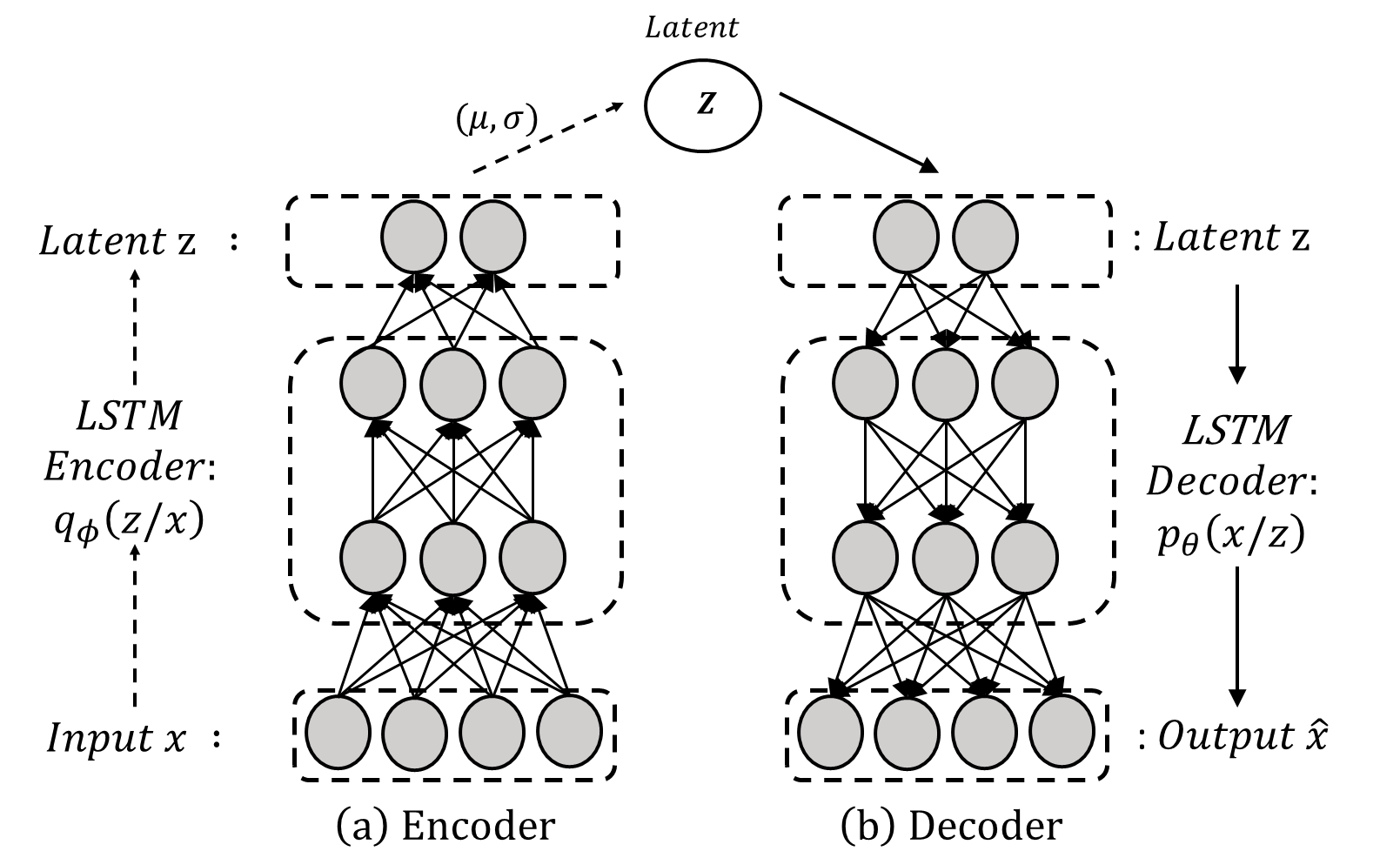}\\
				(a) NAD Event Model.
			\end{minipage}%
			\begin{minipage}{.4\textwidth}
				\centering
				\includegraphics[width=0.9\linewidth]{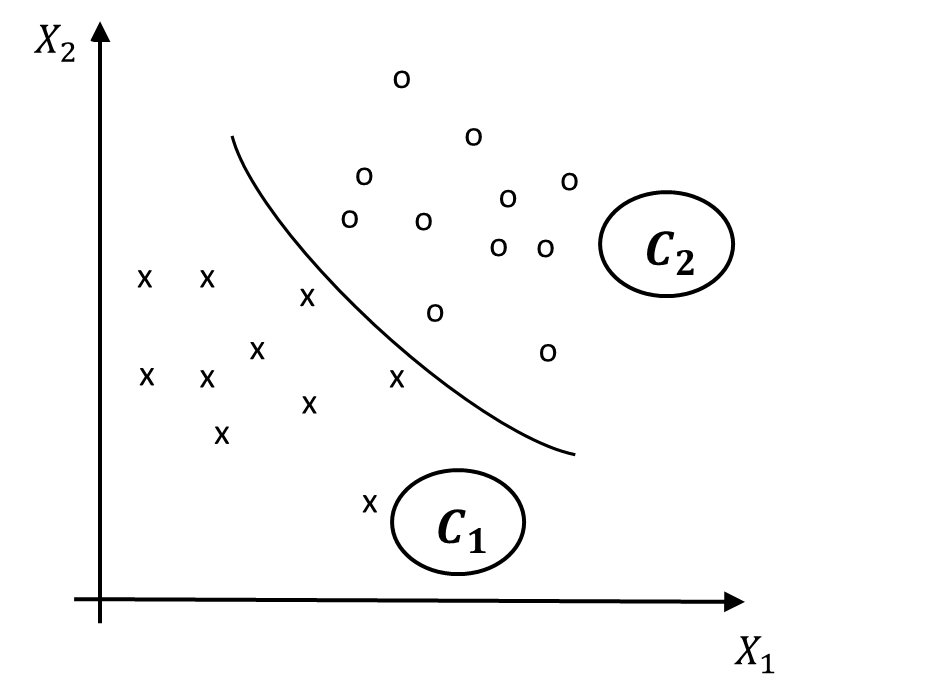}\\
				\vspace{-1mm}
				(b) RF classifier.
				\label{fi:GNBclass}
			\end{minipage}%
			\vspace{-1mm}
			\caption{Selected models.}
			\label{fi:models}
			\vspace{-3mm}
		\end{figure*}

		\subsection{Random Forest} \label{subsection:RF}
		Here, we use the Random Forest algorithm~\cite{rigatti2017random} as our supervised machine learning model $M^\prime$, which assumes that $X_1, X_2,\cdots, X_M$ are $M$ attributes. A sample $S$ is denoted as a vector $(x_1, x_2,\cdots, x_M)$, where $x_i$ is the observed value of $X_i$. We also define the class label \(C=\{c_1,c_2 \}\) consisting of two different values $c_1$ (normal) and $c_2$ (abnormal) whose classification depends on the attributes. Hence, the class predicted by the Random Forest method is given as,
		\begin{equation}
			C_{rf}(S) =\argmax_{c \in C}p(c) \prod_i p(x_i|c), \nonumber
		\end{equation}
		where training samples are employed to predict the $p(x_i|c)$ value. The decision boundary, the curve that separates the points in two groups $c_1$ and $c_2$, is calculated using the classifier function $f_{rf}$:
		\begin{equation}
			f_{rf}(S) = \frac{p(c=c_1)}{p(c=c_2)} \prod_i \frac{p(x_i|c=c_1)}{p(x_i|c=c_2)}
		\end{equation}
		Hence, the network traffic belongs to class $c_1$ if $f_{rf}(S)\geq 1 $, $C_2$ otherwise. 
		Fig.~\ref{fi:models}(b) shows a conceptual view of the RF classifier. 
		
		%
		
		\section{Methodologies} \label{sec:Methodologies}
		\subsection{Two-layer Anomaly Detection Strategy (ITAD-S)}
		ITAD-S consists of a two-layer model, i.e., an Unsupervised Deep Learning model $M$ in the first layer and a Supervised Machine Learning model $M^\prime$\footnote{In this paper, we use LSTM-VAE and RF as an example of the unsupervised deep learning model $M$ and supervised machine learning model $M^\prime$ to better elaborate our work.} in the second layer (as shown in Fig. \ref{fi:framework}). The key idea behind it is that it skillfully leverages the advantages of supervised learning to improve the performance of unsupervised learning models and ultimately construct an efficient unsupervised learning framework for anomaly detection in dynamic environments. This strategy has two critical processes, i.e., a high-confidence pseudo-label generation process and a prediction process for fully automated unsupervised anomaly detection.
		
		\textbf{High-confidence pseudo label generation process}: 
		This process utilizes a novel Adaptive Threshold Calculation technique to select high-confidence labeled samples predicted by the unsupervised learning model from the input data. It is utilized to eliminate manual labeling for autonomous updates. 
		
		Specifically, we can use a limited dataset 
		(e.g., the first day's data or randomly select 1\% from the dataset)
		as our first-round training set to train the unsupervised model $M$. 
		As aforementioned, the upcoming network events are predicted by the trained unsupervised model $M$ to predict their loss values. The network events are then divided into three categories by predetermined thresholds $\{T_1, T_2\}$\footnote{T1 and T2 are determined by the quantiles of the loss distributions for normal and abnormal samples, respectively.} (where $T_1<T_2$): high-confidence anomalous samples (i.e., true positive network events), high-confidence normal samples (i.e., true negative network events), and uncertain samples, as illustrated in Fig. \ref{fi:framework}.
		
		\textbf{Prediction process.}:
		This process predicts anomalies by the machine learning $M^\prime$ trained with the pseudo-labeled high-confidence samples from the previous process. 
		The combination of both the high-confidence pseudo label generation process and the prediction process enhances the adaptive NAD framework's learning ability and ensures a low false positive of the models. It is worth noting that the machine learning model $M^\prime$ (i.e., RF in our work) is employed to make final predictions, and once it detects an abnormal network event, it raises an alert immediately and sends this abnormal data to the system management. 
		
		\begin{figure}[t]
			\centering
			\includegraphics[width=0.95\linewidth]{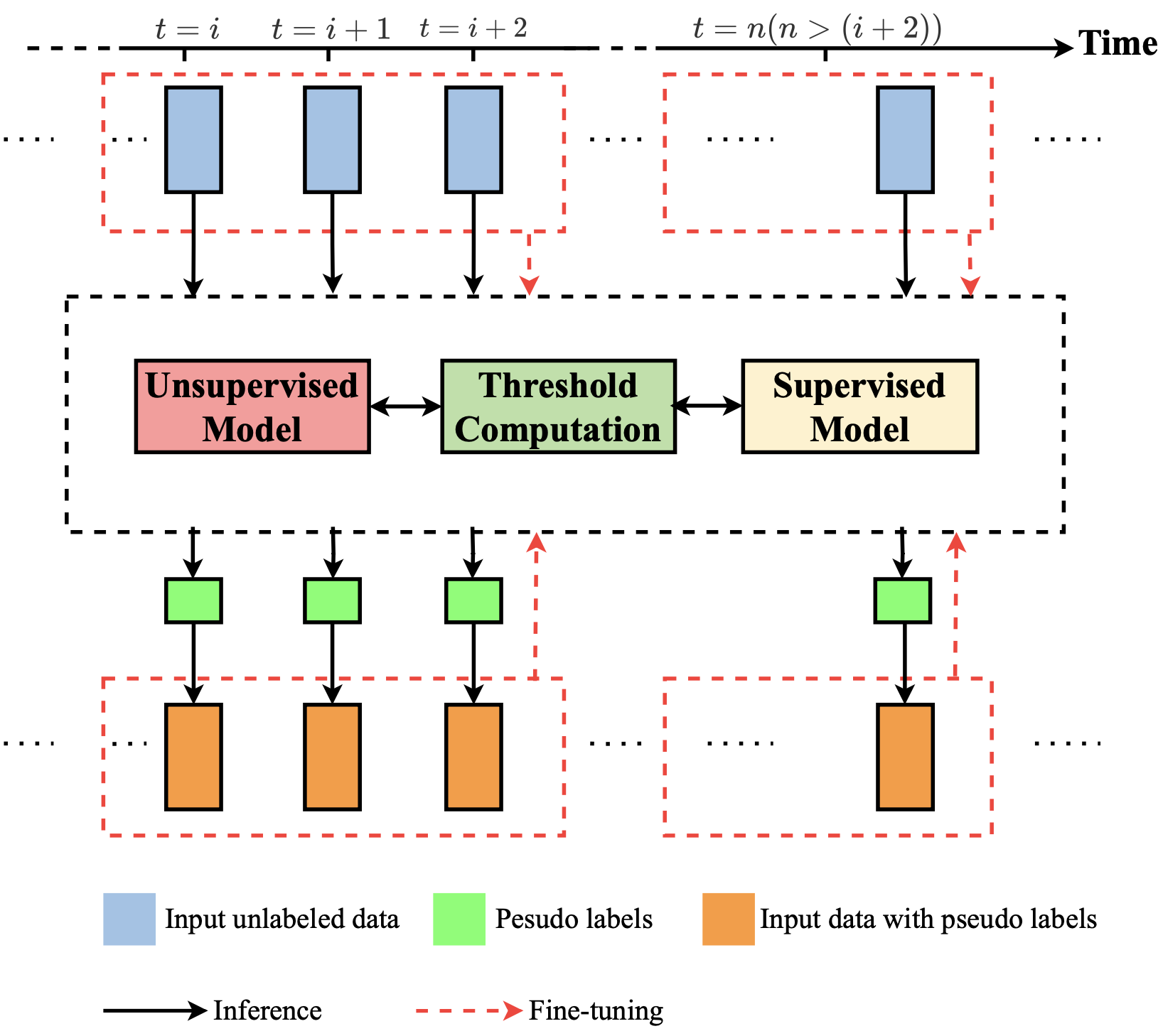}
			\caption{Timeline of Adaptive NAD.}
			\label{fig:onlineLearning}
			\vspace{-4mm}
		\end{figure}
		
		\subsection{\textbf{Online Training Scheme}}
		As illustrated in Fig. \ref{fi:framework}, online training scheme refers to periodically update thresholds $\{T_1, T_2\}$ and the models $\{M, M^\prime\}$ in Adaptive NAD, enabling it to efficiently and accurately predict emerging anomalies (e.g., zero-day attacks) and evolving normal samples (e.g., new normal patterns) in complex real-world environments. 
		
		
		\begin{algorithm}
			\caption{Online Training Scheme} \label{alg:TrainingAdaptiveNAD} 
			\SetKwInOut{Input}{Input}
			\SetKwInOut{Output}{Output}
			\Input{Dataset for the first-round of training: $X_0$;
				Streaming new input vector: \( \{ x_{l_0+1}, \ldots, x_i, \ldots \} \), \( i > l_0, i \in \mathbb{N} \).
			}
			\Output{Well-trained model $M$, $M^\prime$. }
			\BlankLine
			Randomly initialize the unsupervised model $M$\\
			Randomly initialize the machine learning model $M^\prime$;\\
			Train(\(X_0, M\)); \\
			\ForEach{$x_{0i} \in X_0$}{
				\(\text{Loss} \leftarrow M(x_{0i})\); \\
				$LossDB_{\text{normal}} \leftarrow LossDB_{\text{normal}} \cup \{\text{Loss}_{0i}\}$;
			}
			\( T_1 \leftarrow \text{AdaptiveThreshold}(LossDB_{\text{normal}}, p_\text{0}) \);
			
			\While{inputting \( x_i \)}{
				\( \text{Loss} \leftarrow M(x_i) \);\\
				$BatchDB \leftarrow BatchDB \cup \{x_i, \text{Loss}\}$;
				
				\If{\(\text{Size}(LossDB_{\text{abnormal}}) < n\)}{
					$X \leftarrow X \cup \{x_i\}$;\\
					\If{\( \text{Loss} < T_1 \)}{
						$LossDB_{\text{normal}} \leftarrow LossDB_{\text{normal}} \cup \{\text{Loss}\}$;\\
						$Y \leftarrow Y \cup \{\text{normal}\}$;\\
						$X_{\text{normal}} \leftarrow X_{\text{normal}} \cup \{x_i\}$;
					}
					\Else{
						$LossDB_{\text{abnormal}} \leftarrow LossDB_{\text{abnormal}} \cup \{\text{Loss}\}$;\\
						$Y \leftarrow Y \cup \{\text{abnormal}\}$;
					}
					\text{continue};
				}
				
				\If{\( \text{Loss} < T_1 \)}{
					$LossDB_{\text{normal}} \leftarrow LossDB_{\text{normal}} \cup \{\text{Loss}\}$;\\
					$X \leftarrow X \cup \{x_i\}$;\\
					$Y \leftarrow Y \cup \{\text{normal}\}$;\\
					$X_{\text{normal}} \leftarrow X_{\text{normal}} \cup \{x_i\}$;
				}
				\ElseIf{\( \text{Loss} > T_2 \)}{
					$LossDB_{\text{abnormal}} \leftarrow LossDB_{\text{abnormal}} \cup \{\text{Loss}\}$;\\
					$X \leftarrow X \cup \{x_i\}$;\\
					$Y \leftarrow Y \cup \{\text{abnormal}\}$;
				}
				\ElseIf{\( T_1 \leq \text{Loss} \leq T_2 \)}{
					\( y_i \leftarrow M^\prime(x_i) \);\\ 
					\If{\( y_i == normal \)}{
						$X_{\text{normal}} \leftarrow X_{\text{normal}} \cup \{x_i\}$;
					}
				}
				\If{\(\text{Size}(BatchDB) == m\)}{
					\( T_1 \leftarrow \text{AdaptiveThreshold}(LossDB_{\text{normal}}, p_\text{0}) \);\\
					\( T_2 \leftarrow \text{AdaptiveThreshold}(LossDB_{\text{abnormal}}, p_\text{1}) \);\\
					Train(\(X_{\text{normal}}, M\)); \\
					Train(\(X, Y, M^\prime\)); \\
					$Empty(X, Y, X_{\text{normal}}, BatchDB)$.
				}
			}
		\end{algorithm}
		
		\begin{algorithm}
			\caption{Adaptive Threshold} \label{alg:AdaptiveThresholdComputation}
			\SetKwInOut{Input}{Input}
			\SetKwInOut{Output}{Output}
			\Input{The database storing the loss of samples: $LossDB$.\\
				Percentile for distribution: $p$.}
			\Output{Threshold T. }
			\BlankLine
			\If{\( \text{Size}(LossDB_{\text{normal}}) > 0 \)}{
				$FitDist \leftarrow \text{FitBestDistribution}(LossDB)$\;
				$\text{T} \leftarrow \text{ComputePercentile}(FitDist, p)$\;
			}
			\Else{
				$\text{T} \leftarrow \emptyset$.
			}
		\end{algorithm}
		
		As depicted in Fig. \ref{fig:onlineLearning}, the proposed online learning framework follows a timeline that incrementally enlarges the training set with high-confidence pseudo-labels and regularly updates the unsupervised deep learning model $M$ with the high-confidence pseudo-labeled dataset. For example, given time series samples from $t=i$ to $t=i+2$, the loss values are first predicted/inferred using an unsupervised model. Based on these loss values and the predefined thresholds $\{T_1, T_2\}$, high-confidence pseudo-labels are assigned to the samples. The samples and their corresponding labels are then used to update both the supervised and unsupervised models. Simultaneously, the thresholds $\{T_1, T_2\}$ are updated with the new loss values. The updated models and thresholds will be used for predicting the samples in the next time period.
		
		The online learning framework is presented in Algorithm~\ref{alg:TrainingAdaptiveNAD}, where $X$ and $Y$ represent high-confidence samples classified by the unsupervised model, $X_{\text{normal}}$ represents the normal samples used to train the unsupervised model, $n$ represents the number of training samples in the initial phase, $LossDB$ is a fixed-size, first-in-first-out (FIFO) queue, and $m$ represents the number of samples used for model updates. Initially, the unsupervised model $M$ with random initialized is trained on a small unlabeled first-round training dataset $X_0$ (Algorithm~\ref{alg:TrainingAdaptiveNAD} line 3). Then, use $M$ to compute and collect the loss of $X_0$, and compute the threshold $T_1$ (Algorithm~\ref{alg:TrainingAdaptiveNAD} lines 4-7). Subsequently, the online training process comprises two stages: The initial phase: when the abnormal loss is insufficient, only $T_1$ is used to generate pseudo-labels (Algorithm~\ref{alg:TrainingAdaptiveNAD} lines 11-20). The second phase: both $T_1$ and $T_2$ are used to generate pseudo-labels (Algorithm~\ref{alg:TrainingAdaptiveNAD} lines 21-33). At regular intervals, $T_1$ and $T_2$ are recalculated and the model $M$ and $M^\prime$ is trained (Algorithm~\ref{alg:TrainingAdaptiveNAD} lines 34-39).
		The Adaptive Threshold Computation algorithm (Algorithm~\ref{alg:AdaptiveThresholdComputation}) computes an adaptive threshold for sample losses. It fits the best distribution to the data (Algorithm~\ref{alg:AdaptiveThresholdComputation} line 2) and computes the percentile of this distribution as the threshold (Algorithm~\ref{alg:AdaptiveThresholdComputation} line 3).
		
		\subsection{Adaptive Threshold Calculation Technique} \label{subsec:ThresholdCalculation}
		When new network samples arrive, the behavioral patterns of network traffic may change. Consequently, the predefined thresholds for classifying network samples (e.g., normal or abnormal) need to be periodically updated. Therefore, a novel dynamic threshold calculation strategy that automatically learns from the ongoing behaviors and adapts the corresponding thresholds to newer coming data.  
		The initial thresholds $T_1$ and $T_2$ are calculated based on a limited dataset
		(e.g., the data from the first day or randomly selected 1\% from the dataset).  
		We use these samples to train the unsupervised learning model. Initially, we use only one threshold $T_1$. Then, in the second round of training, we use the trained unsupervised model to predict the loss values for new samples. These loss values are divided into normal and abnormal samples using the threshold $T_1$. When a sufficient number of abnormal samples are collected (e.g., when the number of abnormal samples reaches 500), we calculate a second threshold $T_2$. From the third round of training onward, we follow the process described in Section \ref{sec:systemDesign} for training and updating.
		In this way, the threshold for the unsupervised model $M$ can be automatically determined for ongoing network anomaly detection. 
		
		We use the network loss databases to determine the thresholds. Two thresholds, $T_1$ and $T_2$, are utilized to identify true positive and true negative network traffic, which can then be used to train the supervised learning model $M^\prime$ during the online phase. The thresholds are selected as follows:
		\begin{enumerate}
			\item \underline{Phase 1:} 
			The Probability-Probability (P-P) plot is used to find the most suitable loss distributions that fit the datasets. 
			We utilize the CIC-Darknet2020 as an example to prove the effectiveness of the proposed threshold calculation technique. To show its generality, we also illustrate details of the threshold calculation on the CIC-DoHBrw-2020 dataset in Appendix A.  
			Fig. \ref{fi:networkN01} illustrates the empirical distribution of the losses obtained from the Adaptive NAD model against the best-fitting theoretical distributions. By comparing several classical loss distributions, the log normal distribution is selected as the optimal distribution to determine the thresholds.
			\item \underline{Phase 2:} MLE is used to predict the parameters of the log normal distribution by maximizing a likelihood function. The estimation of parameters $\mu$ and $\sigma$ using the maximum likelihood technique from a sample of $n$ observations, $X=(x_1,x_2,...,x_n)$ being X the loss observations, is calculated as follows:
			
			\begin{theorem}
				Let $X \sim \text{lognormal}(\mu, \sigma^2)$. The maximum likelihood estimators under $n$ observations of the parameters are:
			\end{theorem}
			
			\begin{itemize}
				\item[(i)] \(
				\hat{\mu} = \frac{\sum_j \ln(x_j)}{n}
				\)
				\item[(ii)] \(
				\hat{\sigma}^2 = \frac{\sum_j (\ln(x_j) - \hat{\mu})}{n}
				\)
			\end{itemize}

			\begin{proof}\renewcommand{\qedsymbol}{}
				\begin{equation}
					\begin{split}
						\label{eq:mle}
						L(\mu,\sigma^2|X) 
						&= \prod_{j = 1}^{n}[f(x_j|\mu,\sigma^2)], \\
						&={(2\pi\sigma^2)}^{-n/2} \prod_j {x_j}^{-1} \\
						&exp(\sum_j-\frac{(log(x_j)-\mu)^2}{2\sigma^2}). \nonumber
					\end{split}
				\end{equation}
				Applying it to Eq.~\eqref{eq:mle} to have the log likelihood function $l$, which is calculated as:
				\begin{align}
					&l(\mu,\sigma^2|X) \nonumber \\
					& =-\frac{n}{2}\ln(2\pi\sigma^2)
					-\sum_j \ln(x_j)-\frac{\sum_j \ln(x_j)^2}{2\sigma^2} \nonumber \\
					&\ \ \ \  +\frac{\sum_j \ln(x_j)\mu}{\sigma^2}
					-\frac{n\mu^2}{2\sigma^2}.\nonumber
				\end{align}
				
				Next, we maximize $l$ by applying the derivative in terms of $\mu$ and $\sigma$. $\hat{\mu}$ value is obtained by:
				
				\begin{equation}
					\frac{\partial l}{\partial \mu}=\frac{\sum_j \ln(x_j)}{\sigma^2}-\frac{n\hat{\mu}}{\sigma^2}=0,\nonumber
				\end{equation}
				\begin{equation}
					\hat{\mu}=\frac{\sum_j \ln(x_j)}{n}.\nonumber
				\end{equation}
				The value of $\hat{\sigma}$ is calculated using the following:
				\begin{align}
					\label{eq:sigma}
					\frac{\partial l}{\partial \sigma^2} =-\frac{\sum_j(\ln(x_j)-\hat{\mu})^2}{2(\hat{\sigma}^2)^2}+\frac{\sum_j(\ln(x_j)-\hat{\mu})^2}{2(\hat{\sigma}^2)^2}. \nonumber
				\end{align}
				Then, letting $\frac{\partial l}{\partial \sigma^2}=0$, we have $\hat{\sigma}^2$ as:
				\begin{equation}
					\hat{\sigma}^2=\frac{\sum_j(\ln(x_j)-\hat{\mu})}{n}. \nonumber
				\end{equation}
			\end{proof}
			
			\item \underline{Phase 3:}
			
			\begin{figure}[t]
				\centering
				\begin{minipage}{.24\textwidth}
					\scalebox{1.}{\includegraphics[width=1.\linewidth]{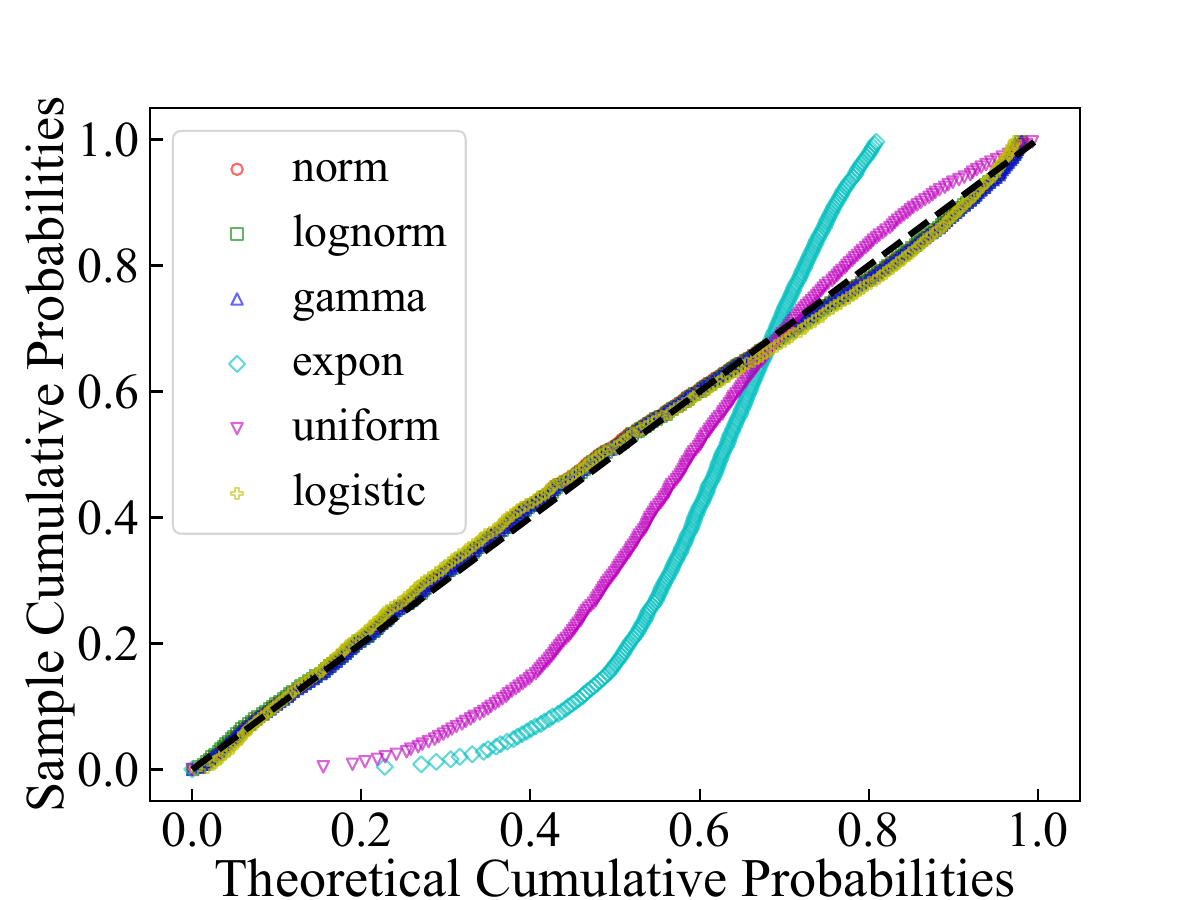}} \\
					\vspace{-1mm}
					(a) Normal Distribution.
				\end{minipage}%
				\begin{minipage}{.24\textwidth}
					\scalebox{1}{\includegraphics[width=1\linewidth]{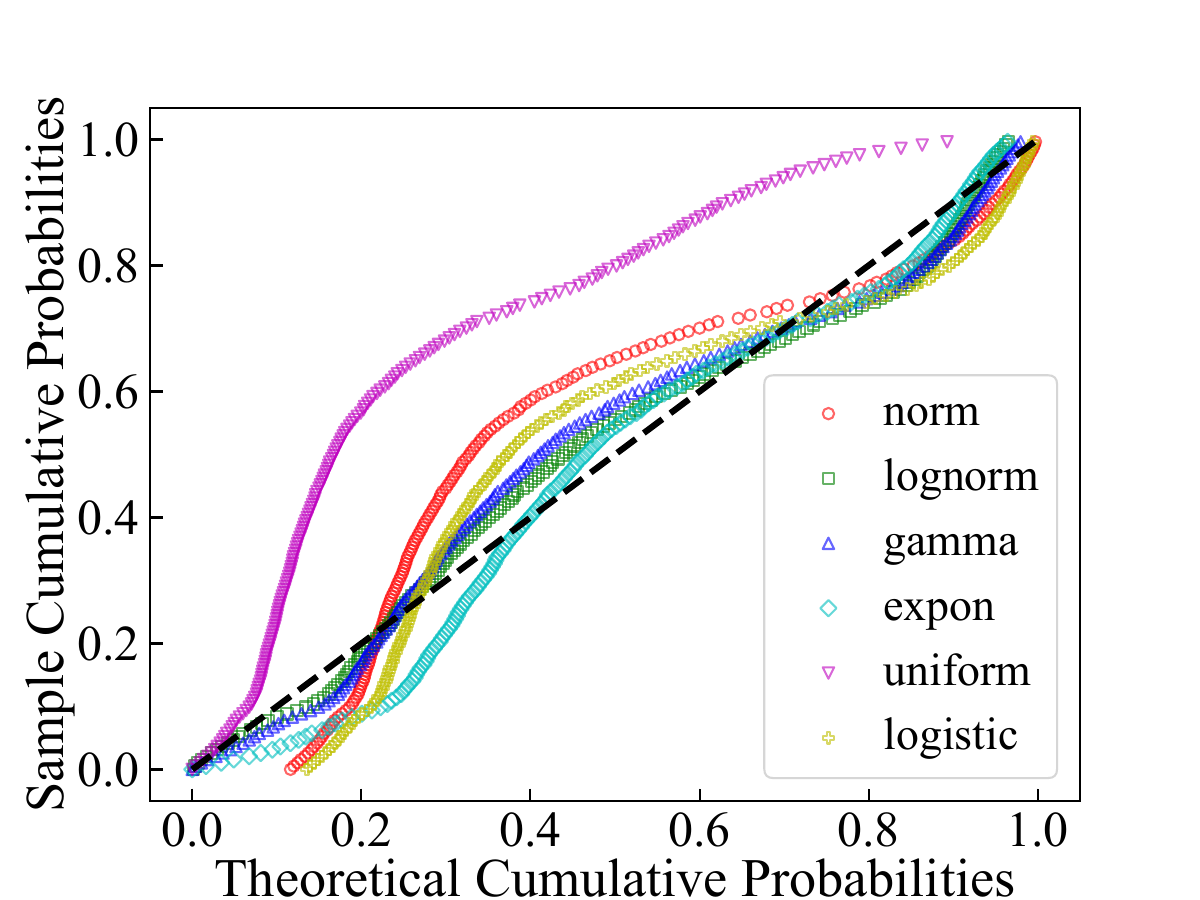}} \\
					\vspace{-1mm}
					(b) Abnormal Distribution.
				\end{minipage}%
				\caption{Loss distributions of normal and abnormal network traffic for CIC-Darknet2020.}
				\label{fi:networkN01}
				\vspace{-2mm}
			\end{figure}
			
			\begin{figure}[t]
				\centering
				\begin{minipage}{.24\textwidth}
					\centering
					\scalebox{.23}{\includegraphics{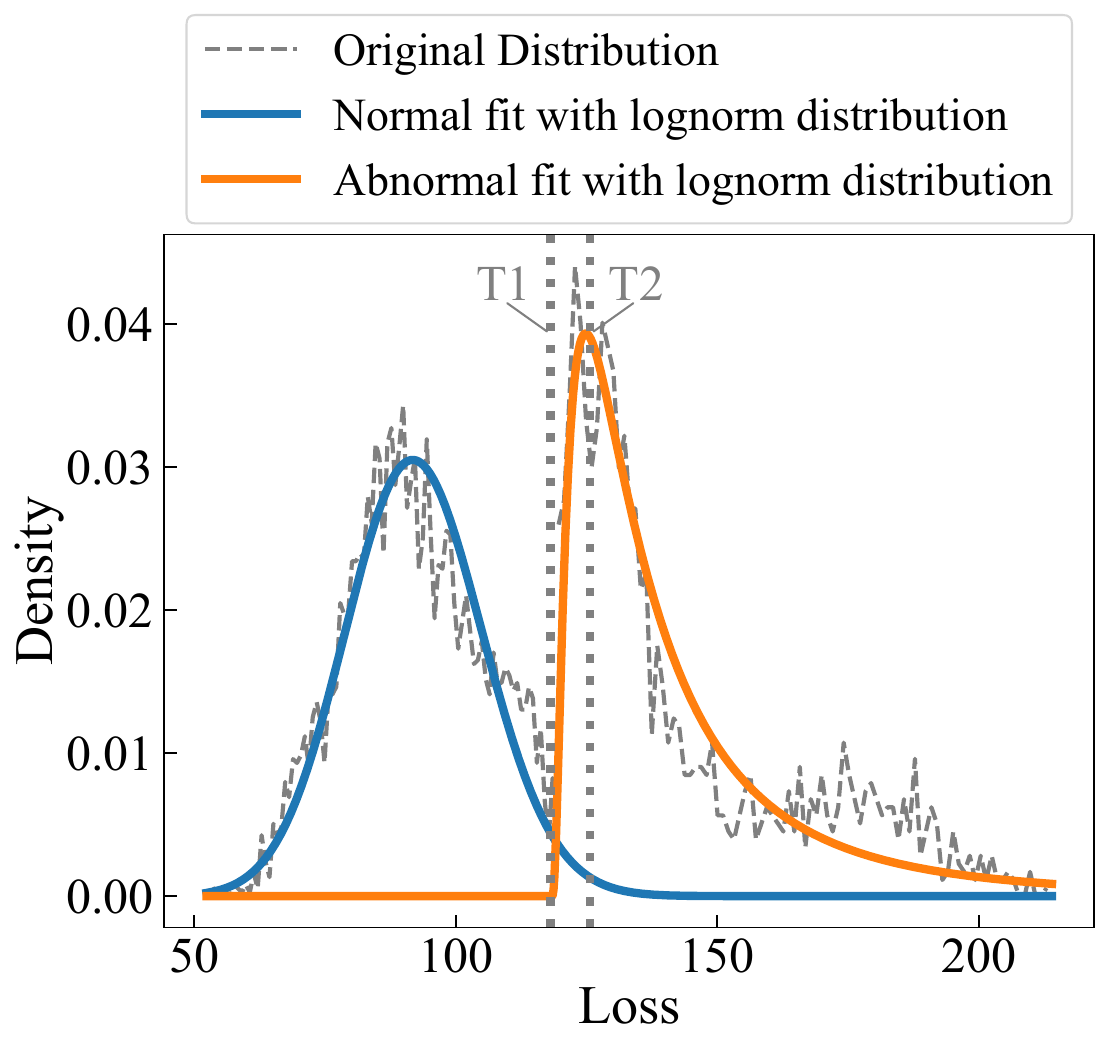}} \\
					(a) Network Distributions.
				\end{minipage}%
				\begin{minipage}{.24\textwidth}
					\centering
					\scalebox{.21}{\includegraphics{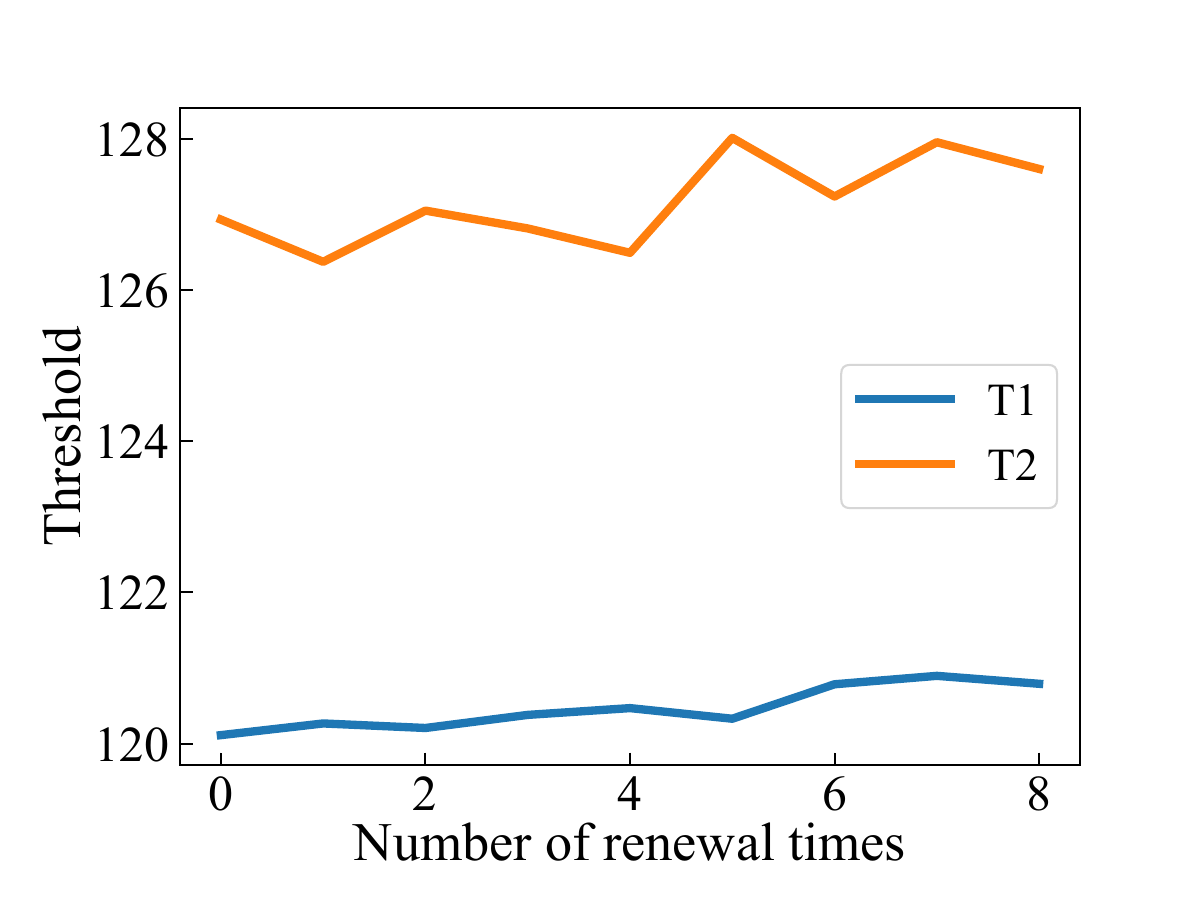}} \\
					(b) Threshold update.
				\end{minipage}
				\caption{Adaptive threshold calculation on CIC-Darknet2020.}
				\label{fi:networkN02}
				\vspace{-2mm}
			\end{figure}

			The threshold $T_i$ is obtained by Proposition \ref{prop:threshold}. 
			\begin{prop}
				\label{prop:threshold}
				Let $T_1$ and $T_2$ be the thresholds for network traffic, and $p_1$ and $p_2$ are the percentiles for normal and abnormal loss distributions related to $T_1$ and $T_2$. If the normal and abnormal loss follows a log-normal distribution with parameters $\mu_i$ and $\sigma_i$, being $i=1$ for normal and $i=2$ for abnormal traffic, then $T_1$ and $T_2$ are calculated as follows with a ($1-\alpha$) confidence level ($0<\alpha<1$):
				\begin{equation}
					T_1 = exp(\Phi^{-1}(p_1)\sigma_1+\mu_1), \nonumber
					\label{eq:lognormal}
				\end{equation}
				\begin{equation}
					T_2 = exp(\Phi^{-1}(1-p_2)\sigma_2+\mu_2) .\nonumber
				\end{equation}
			\end{prop}
			
			\begin{proof}\renewcommand{\qedsymbol}{}
				Let $X_1$ and $X_2$ be the loss of normal observations and loss of abnormal observations with cumulative distribution functions $F_1$ and $F_2$, respectively, formulated as follows:
				\begin{equation}
					F_1(T_1)=P(X_1\leq T_1)= p_1,\nonumber
				\end{equation}
				\begin{equation}
					F_2(T_2)=P(X_2\geq T_2)= p_2.\nonumber
				\end{equation}
				
				By assumption, $F_{i}(T_i)$ follows a log-normal distribution that is defined as follows:
				\begin{align}
					F_{1}(T_1)&=P(X_1\leq T_1)\nonumber\\&=\int_{0}^{T_1}f_1(x_1)dx_1=\Phi\left(\frac{log(T_1)-\mu_1}{\sigma_1}\right),\nonumber
				\end{align}
				\begin{align}
					F_{2}(T_2)&=1-P(X_2\leq T_2)\nonumber\\&=1-\int_{0}^{T_2}f_2(x_2)dx_2=1-\Phi\left(\frac{log(T_2)-\mu_2}{\sigma_2}\right),\nonumber
				\end{align}
				where $\Phi$ is the cumulative distribution function of the standard normal distribution, i.e., $N(0,1)$.  $f_1(x_1)$ and $f_2(x_2)$ are the probability density functions of the log-normal distribution, which is calculated by:
				\begin{align}
					f_1(x_1)= f_2(x_2)=\frac{1}{x\sqrt{2\pi}\sigma}exp(-\frac{(log(x)-\mu)^2}{2\sigma^2}).\nonumber
				\end{align}
				Based on the above equations, the threshold $T_{1}$ and $T_{2}$ are calculated as follows:
				\begin{equation}
					p_1 = \Phi\left(\frac{log(T_1)-\mu_1}{\sigma_1}\right),\nonumber
				\end{equation}
				\begin{equation}
					p_2 = 1-\Phi\left(\frac{log(T_2)-\mu_2}{\sigma_2}\right),\nonumber
				\end{equation}
				\begin{equation}
					\Phi^{-1}(p_1) =\frac{log(T_1)-\mu_1}{\sigma_1},\nonumber
				\end{equation}
				\begin{equation}
					\Phi^{-1}(1-p_2) =\frac{log(T_2)-\mu_2}{\sigma_2},\nonumber
				\end{equation}
				\begin{equation}
					T_1 = exp(\Phi^{-1}(p_1)\sigma_1+\mu_1),\nonumber
				\end{equation}
				\begin{equation}
					T_2 = exp(\Phi^{-1}(1-p_2)\sigma_2+\mu_2).\nonumber
				\end{equation}
			\end{proof}
		\end{enumerate}
		In our model, the thresholds $T_1$ and $T_2$ are determined based on the quantiles of the best-fit statistical distributions for normal and abnormal losses, respectively. Specifically, $T_1$ should be set at a high percentile of the best-fit distribution for normal losses, while $T_2$ should be set at a low percentile of the best-fit distribution for abnormal losses. Selecting a high percentile for $T_1$ ensures that most normal samples, which typically have lower loss values, fall below it, thereby reducing false positives. Conversely, setting $T_2$ at the low percentile of the best-fit distribution for abnormal losses ensures that most abnormal samples, which typically have higher loss values, exceed it, minimizing false negatives.
		
		The description of renewed thresholds of the proposed adaptive NAD model for both $T_1$ and $T_2$ over 8 renewal times is given in Fig. \ref{fi:networkN02}(b),
		where up to 5000 normal and 5000 abnormal loss values are retained in every renewal time and the old ones are discarded as new ones are added, ensuring that the most recent samples are always maintained. This approach allows the model to use the most up-to-date data for accurately determining the thresholds $T_1$ and $T_2$. 
		We observe that the recent characteristics of network traffic determine threshold variations, and $M$'s loss value is from 50 to 230. In addition, the thresholds for $T_1$ and $T_2$ range from 120 - 122 and 126 - 128, respectively. 
		
		\subsection{Online Inferencing}
		As shown in Fig. \ref{fi:framework}, the online inferencing refers to the usage of the well-trained network anomaly detector (i.e., the trained supervised machine learning model $M^\prime$) from the online training for real-time network anomaly detection. 
		In our setting, the online training is conducted continuously as streaming data arrive, and the online inference always uses the up-to-date trained supervised machine learning model $M^\prime$ for real-time network anomaly prediction. 
		
		\subsection{Complexity and storage analysis}
		\textbf{Model complexity.}
		The model complexity of the online training of Adaptive NAD is determined by the most complex component of the model. For instance, in this work, Adaptive NAD uses the threshold computation technique and two models, including LSTM-VAE (with an input sequence length of 5, hidden size of 64, latent size of 32, and a single layer) and RF. The LSTM-VAE is the most complex model, with 0.26M FLOPs and 58,207 parameters, which defines the model complexity of Adaptive NAD.
		\textcolor{black}{The model complexity of the online inferencing of Adaptive NAD is decided by the used supervised machine learning model $M^\prime$, for example, if RF is used as $M^\prime$, then its time complexity for inferencing is $O(k \times d)$ and its space complexity if $O(k\times 2^d)$, where $k$ is the number of trees and $d$ is the maximum depth of each tree.}
		
		\textbf{Storage cost.} As for the space complexity, the primary cost includes two loss databases with fixed-size (e.g., 5000 in this work), FIFO queues, a temporary database from each batch that is periodically cleared, and all anomalies to assist security analysis for network operators.
		
		\section{Evaluation} \label{sec:Evaluation}
		This section presents the experimental setup and comprehensive results. 
		We first detail the datasets in Section~\ref{subsec:dataset}, 
		baselines in Section~\ref{subsec:baselines}, 
		evaluation metrics in Section~\ref{subsec:metrics}, 
		and implementation details in Section~\ref{subsec:impl}. 
		Subsequently, we analyze the overall performance in Section~\ref{subsec:overall performance} 
		and discuss the ablation studies in Section~\ref{subsec:ablation}.

		\subsection{Experimental Setup}
		\subsubsection{Datasets}
		\label{subsec:dataset}
		To comprehensively evaluate Adaptive NAD across diverse network security scenarios, we conduct experiments on three widely-used benchmark datasets from different network environments: CIC-Darknet2020~\cite{habibi2020didarknet} (encrypted Tor traffic), NSL-KDD~\cite{tavallaee2009detailed} (general network intrusion), and Edge-IIoTset~\cite{mbc1-1h68-22} (IoT/IIoT security). This combination ensures our evaluation spans different traffic patterns, 
		attack types, and network architectures.
		
		\textbf{CIC-Darknet2020.} 
		The CIC-Darknet2020~\cite{habibi2020didarknet} is a two-layered dataset (ISCXTor2017 and ISCXVPN2016) capturing both VPN and Tor encrypted traffic from Darknet's hidden services. We evaluate Adaptive NAD on the Tor and non-Tor binary classification task, which contains 93,309 normal and 1,392 abnormal samples collected over 13 days with 81 features across 8 network types, exhibiting a highly imbalanced distribution with a ratio of 67:1.
		
		\textbf{NSL-KDD.} 
		Following the experimental setup in~\cite{AOCIDS}, we use the 
		NSL-KDD~\cite{tavallaee2009detailed} dataset, a standard benchmark for 
		intrusion detection. The dataset comprises 148,517 
		samples with 121 features after preprocessing, covering four attack categories (DoS, U2R, R2L, and Probing) with 62 attack types.
		
		\textbf{Edge-IIoTset.} 
		The Edge-IIoTset~\cite{mbc1-1h68-22} is specifically designed for evaluating 
		intrusion detection in IoT and IIoT environments. It includes 14 attack types 
		related to IoT connectivity protocols, with 2,154,864 benign and 1,076,200 
		attack samples extracted into 81 features, resulting in an imbalance ratio of 2:1.
		
		\subsubsection{Baselines}
		\label{subsec:baselines}
		We compare Adaptive NAD with nine baseline methods in the same 
		online learning setting.
		
		\textbf{Classical Unsupervised Methods.}
		We adapt four widely-used anomaly detection algorithms using the River library: 
		Half-Space Trees (HST)~\cite{tan2011fast} for tree-based online outlier scoring, 
		KMeans clustering~\cite{kanungo2002efficient} that treats distant points as anomalies, 
		One-Class SVM (OCSVM)~\cite{aguayo2018novelty} that learns a boundary around normal data, 
		and Self-Organizing Maps (SOM)~\cite{aguayo2018novelty} that detect anomalies via neuron distance.
		
		\textbf{Deep Learning Baselines.}
		We compare against five state-of-the-art methods:
		\begin{itemize}
			\item \textbf{AE\_LOF}~\cite{odiathevar2021online}: Employs a two-stage filtering mechanism 
			where samples first pass through reconstruction error (RE) thresholding, 
			and borderline cases are further verified by Local Outlier Factor (LOF) 
			in the latent space, with online threshold adaptation via Mann-Whitney test.
			
			\item \textbf{AOC-IDS}~\cite{AOCIDS}: Leverages contrastive learning 
			with Cluster-Repelling Contrastive (CRC) loss and dual-view decision-making 
			(encoder/decoder feature spaces) via Gaussian mixture fitting, 
			supporting pseudo-label-driven online adaptation with random label flipping.
			
			\item \textbf{DAGMM}~\cite{zong2018deep}: Jointly optimizes deep autoencoding 
			and Gaussian Mixture Model (GMM) estimation in an end-to-end framework, 
			using both reconstruction features and latent representations for density-based scoring.
			
			\item \textbf{GDN}~\cite{deng2021graph}: Employs graph neural networks 
			to model inter-feature dependencies in multivariate time series, 
			detecting anomalies via learned graph-structured deviation patterns.
			\item \textbf{DeepAID}~\cite{han2021deepaid}: Utilizes deep attention mechanisms 
			to capture temporal dependencies in network traffic sequences, 
			combining LSTM encoder with attention-weighted reconstruction error 
			for enhanced anomaly detection in streaming data.
		\end{itemize}
		
		\textbf{Implementation of Baselines.}
		All baseline methods are evaluated under identical conditions with 
		the same preprocessing pipeline, online learning configuration, and 
		hardware environment to ensure fair comparison. Classical methods 
		(HST, KMeans, OCSVM, SOM) are implemented using the River library~\cite{montiel2021river}. 
		Deep learning baselines use the architectures reported in their original papers: 
		GDN (2 layers), DAGMM (4 layers), AOC-IDS (4 layers), DeepFT (6 layers), 
		DeepAID (2 layers), and OOF (4 layers). 
		
		\subsubsection{Evaluation Metrics}
		\label{subsec:metrics}
		We use seven metrics to comprehensively evaluate model performance: 
		Accuracy (Acc.), Precision-macro (Pre.), Recall-macro (Rec.), 
		F1-score-macro (F1), False Alarm Rate (FAR), Missed Detection Rate (MDR), 
		and Area Under the ROC Curve (AUC).
		
		FAR (equivalent to the False Positive Rate) measures the proportion of normal 
		traffic wrongly classified as abnormal, while MDR (equivalent to the False Negative Rate) 
		indicates the proportion of anomalous traffic missed by the model. 
		AUC summarizes the trade-off between true positive rate and false positive rate 
		across all possible decision thresholds, providing an overall measure of 
		discrimination ability. Higher values of Acc., Pre., Rec., F1, and AUC indicate 
		better performance, while lower FAR and MDR are preferred.
		
		\subsubsection{Implementation Details}
		\label{subsec:impl}
		We describe the data preprocessing, online learning configuration, 
		model architecture, and experimental environment used in our experiments.
		
		\textbf{Data Preprocessing.}
		Our preprocessing pipeline is straightforward and avoids complex feature 
		engineering. It includes three steps. First, we remove metadata columns 
		(Flow ID, Timestamp, and Attack Name where applicable) and convert IP 
		addresses to integer representations. Second, we replace infinite values 
		with NaN and remove entries with missing values. Third, we apply feature 
		normalization using dataset-specific scalers: MinMaxScaler with range 
		$[-1, 1]$ for CIC-Darknet2020 and NSL-KDD, and RobustScaler 
		for Edge-IIoTset. Following these steps, CIC-Darknet2020 has 81 
		features, NSL-KDD has 121 features (with categorical variables 
		already one-hot encoded), and Edge-IIoTset has 81 features. 
		These preprocessing steps are applied consistently across all comparative 
		methods to ensure fair evaluation.
		
		\textbf{Online Learning Configuration.}
		The experiments are conducted in an online learning setting, where the 
		system is initialized with a limited proportion of labeled normal samples 
		and then autonomously adapts to incoming data streams. Specifically, we 
		allocate 5\% of normal data (610 samples) for CIC-Darknet2020, 1\% (673 
		samples) for NSL-KDD, and 1\% (32,310 samples) for Edge-IIoTset as the 
		pre-training set, while the remaining samples form the online streaming 
		set for continuous evaluation and incremental learning. The model undergoes 
		periodic updates every 1,000 samples, with initial training running for 
		500 epochs (early stopping patience=10) and each subsequent update for 10 epochs.
		
		\textbf{Model Architecture and Hyperparameters.}
		The LSTM-VAE architecture consists of 2 LSTM layers with 128 hidden units 
		and a latent dimension of 64 across all datasets. The decoder includes an 
		additional fully connected layer for reconstruction. We use the Adam optimizer 
		with a learning rate of 0.001 and a batch size of 128 for VAE training. 
		The Random Forest classifier is implemented using the River library~\cite{montiel2021river} 
		with 100 trees, enabling incremental learning through one-sample-at-a-time updates. 
		\textcolor{black}{Adaptive thresholds $T_1$ and $T_2$ are initially computed from the 5th and 95th 
			percentiles of the normal loss distribution during pre-training.} Once 500 
		abnormal samples are accumulated in the online phase, the system transitions 
		to dual-distribution fitting based on both normal and abnormal loss distributions.
		
		\textbf{Experimental Environment.}
		All experiments are conducted on a Linux server with an AMD EPYC 
		7282 16-Core Processor at 2.80 GHz and 377 GB of RAM. We implement 
		Adaptive NAD in Python using PyTorch and Scikit-learn libraries. 
		The reported results are averaged over 10 independent runs to ensure 
		statistical reliability.
		
		\subsection{Overall Performance}
		\label{subsec:overall performance}
		We compare AdaptiveNAD against four classical unsupervised methods 
		and five deep learning baselines under the same online learning setting 
		with limited initial labeled data. Tables~\ref{tab:classical_darknet}--\ref{tab:classical_iiotset} 
		present the comprehensive results across all three datasets.
		
		\begin{table}[t]
			\caption{Comparison with classical unsupervised methods on CIC-Darknet2020 (\%).}
			\label{tab:classical_darknet}
			\centering
			\setlength\tabcolsep{1pt}
			\begin{tabular}{ccccccccc}
				\toprule
				Method & Acc. & Pre. & Rec. & F1 & FAR & MDR & AUC & Lat.($\mu$s) \\
				\midrule
				HST & 96.45 & 67.42 & 81.02 & 63.93 & 2.72 & 35.24 & 81.02 & 159.03 \\
				KMeans & 25.16 & 50.99 & 61.71 & 20.84 & 75.99 & \textbf{0.59} & 61.71 & 1470.90 \\
				OCSVM & 98.19 & 50.46 & 51.38 & 50.64 & \textbf{0.29} & 96.95 & 51.38 & 7196.60 \\
				SOM & 76.96 & 54.25 & 87.88 & 49.31 & 23.09 & 1.15 & 87.88 & \textbf{34.19} \\
				AdaptiveNAD & \textbf{98.24} & \textbf{75.51} & \textbf{92.12} & \textbf{76.74} & 1.33 & 12.21 & \textbf{93.00} & 260.33 \\
				\bottomrule
			\end{tabular}
		\end{table}
		
		\begin{table}[t]
			\caption{Comparison with classical unsupervised methods on NSL-KDD (\%).}
			\label{tab:classical_nslkdd}
			\centering
			\setlength\tabcolsep{1pt}
			\begin{tabular}{ccccccccc}
				\toprule
				Method & Acc. & Pre. & Rec. & F1 & FAR & MDR & AUC & Lat.($\mu$s) \\
				\midrule
				HST & 80.84 & 86.02 & 80.20 & 79.87 & \textbf{0.66} & 38.95 & 80.20 & 159.48 \\
				KMeans & 75.36 & 81.72 & 75.89 & 74.21 & 46.22 & \textbf{2.00} & 75.89 & 2266.70 \\
				OCSVM & 50.69 & 25.86 & 49.96 & 33.79 & 28.67 & 71.41 & 49.96 & 10979.80 \\
				SOM & 86.22 & 88.09 & 86.60 & 86.12 & 24.56 & 2.24 & 86.60 & \textbf{37.05} \\
				AdaptiveNAD & \textbf{87.56} & \textbf{89.64} & \textbf{87.27} & \textbf{87.20} & 0.71 & 24.74 & \textbf{87.27} & 251.97 \\
				\bottomrule
			\end{tabular}
		\end{table}
		
		\begin{table}[t]
			\caption{Comparison with classical unsupervised methods on Edge-IIoTset (\%).}
			\label{tab:classical_iiotset}
			\centering
			\setlength\tabcolsep{1pt}
			\begin{tabular}{ccccccccc}
				\toprule
				Method & Acc. & Pre. & Rec. & F1 & FAR & MDR & AUC & Lat.($\mu$s) \\
				\midrule
				HST & 45.24 & 38.86 & 38.78 & 38.68 & 41.48 & 80.91 & 38.75 & 137.47 \\
				KMeans & 79.04 & 80.48 & 83.89 & 78.68 & 31.01 & 1.14 & 83.87 & 1488.88 \\
				OCSVM & 33.75 & 17.10 & 49.92 & 25.27 & 99.57 & \textbf{0.54} & 49.92 & 7431.43 \\
				SOM & 89.94 & 89.25 & 89.19 & 88.95 & 8.57 & 12.98 & 89.17 & \textbf{36.96} \\
				AdaptiveNAD & \textbf{93.71} & \textbf{95.59} & \textbf{90.70} & \textbf{92.55} & \textbf{0.08} & 18.52 & \textbf{90.70} & 260.33 \\
				\bottomrule
			\end{tabular}
		\end{table}
		
		\subsubsection{Comparison with Classical Unsupervised Methods}
		
		\textbf{Overall Performance.} AdaptiveNAD consistently outperforms all classical baselines across the three evaluated datasets. specifically on CIC-Darknet2020, our method achieves an accuracy of 98.24\% and an F1-score of 76.74\%, surpassing the best-performing baseline, HST, by margins of 1.79\% and 12.81\%, respectively. This performance advantage extends to NSL-KDD and Edge-IIoTset: on NSL-KDD, AdaptiveNAD attains 87.56\% accuracy (1.34\% higher than SOM) and an 87.20\% F1-score; on Edge-IIoTset, it reaches 93.71\% accuracy (3.77\% higher than SOM) with a 92.55\% F1-score. The consistent superiority across these diverse network environments—ranging from encrypted Tor traffic to general network intrusions and IoT/IIoT scenarios—underscores the robustness and versatility of our dual-threshold adaptive mechanism.
		
		\textbf{The FAR-MDR Trade-off.} Classical methods exhibit a severe imbalance between false alarm rate (FAR) and missed detection rate (MDR), often optimizing one metric at the expense of the other. OCSVM demonstrates extreme conservatism on CIC-Darknet2020, achieving a low FAR of 0.29\% but an excessively high MDR of 96.95\%, failing to detect the vast majority of anomalies. Conversely, it becomes overly aggressive on Edge-IIoTset with a FAR of 99.57\% and an MDR of 0.54\%, misclassifying nearly all normal traffic as anomalous. KMeans shows the opposite extreme: it achieves minimal MDR (0.59\%--1.14\%) by over-flagging samples, resulting in a prohibitive FAR (31.01\%--75.99\%) that is impractical for real-world applications. HST also lacks stability; although it maintains a low FAR on NSL-KDD (0.66\%), its FAR rises to 41.48\% on Edge-IIoTset, and its MDR remains high (35.24\%--80.91\%) across all datasets. In contrast, AdaptiveNAD balances this trade-off effectively, maintaining a consistently low FAR (0.08\%--1.33\%) alongside a significantly reduced MDR (12.21\%--24.74\%). This confirms that our dual-threshold mechanism successfully navigates the sensitivity-specificity trade-off across diverse traffic distributions.
		
		\textbf{Efficiency vs. Effectiveness.} Tables~\ref{tab:classical_darknet}--\ref{tab:classical_iiotset} also incorporate the inference latency comparison across all datasets. To ensure practical relevance, all measurements are conducted under a strict single-sample inference paradigm. This approach accurately mirrors real-world streaming environments, as accumulating network flows for batch inference introduces latency overheads that are prohibitive for timely threat mitigation. While SOM achieves the lowest latency ($\approx$34--37 $\mu$s per sample), its F1-scores range from 49.31\% to 88.95\%, substantially lower than AdaptiveNAD's consistent performance (76.74\%--92.55\%). On the other extreme, OCSVM and KMeans incur prohibitive computational costs: OCSVM requires 7,196--10,979 $\mu$s per sample, and KMeans demands 1,470--2,266 $\mu$s, rendering them impractical for high-throughput traffic streams despite their varied detection capabilities. AdaptiveNAD maintains a competitive inference speed of approximately 250--260 $\mu$s per sample, which is 27--43$\times$ faster than OCSVM while achieving superior F1-scores and AUC (90.70\%--93.00\%). This efficiency translates to a theoretical single-thread throughput of approximately 3,800--4,000 samples per second, demonstrating that our method achieves an optimal balance between robust detection capability and real-time responsiveness for production network monitoring.
		
		\textbf{Adaptability and Root Cause Analysis.} The limitations of classical methods stem primarily from their rigid decision boundaries and limited capacity to adapt to evolving traffic distributions. OCSVM and HST rely on geometric assumptions that struggle to generalize across diverse network environments, leading to the observed instability in FAR and MDR. Similarly, while SOM offers speed, its shallow topological mapping is often insufficient to capture complex, high-dimensional anomalous patterns without continuous refinement. In contrast, AdaptiveNAD's superiority lies in its robust online adaptability. By dynamically updating the feature memory bank and adjusting dual decision thresholds in an asynchronous update cycle, our model effectively accommodates varying traffic dynamics. This capability ensures robust generalization across encrypted traffic (Darknet), network intrusions (NSL-KDD), and IoT scenarios (Edge-IIoTset), addressing the fundamental lack of plasticity that hinders classical unsupervised approaches.
		
		\begin{table}[t]
			\caption{Comparison with deep learning baselines on CIC-Darknet2020 (\%).}
			\label{tb:comparisonLatestonDarkNet}
			\centering
			\setlength\tabcolsep{1.5pt}
			\begin{tabular}{cccccccc}
				\toprule
				Method & Acc. & Pre. & Rec. & F1 & FAR & MDR & AUC \\
				\midrule
				AE-LOF & 84.13 & 51.19 & 46.45 & 47.37 & 14.89 & 92.22 & 46.45 \\
				AOC-IDS & 64.83 & 52.50 & 77.81 & 42.71 & 35.20 & 9.17 & 77.81 \\
				DAGMM & 93.79 & 54.32 & 79.38 & 55.43 & 5.62 & 35.63 & 79.38 \\
				DeepAID & 97.37 & 72.48 & 91.44 & 72.35 & 2.28 & 12.62 & 90.33 \\
				GDN & 86.63 & 54.45 & 70.00 & 52.59 & 13.18 & \textbf{3.52} & 51.17 \\
				AdaptiveNAD & \textbf{98.24} & \textbf{75.51} & \textbf{92.12} & \textbf{76.74} & \textbf{1.33} & 12.21 & \textbf{93.00} \\
				\bottomrule
			\end{tabular}
		\end{table}
		
		\begin{table}[t]
			\caption{Comparison with deep learning baselines on NSL-KDD (\%).}
			\label{tb:comparisonLatestonNSL-KDD}
			\centering
			\setlength\tabcolsep{1.5pt}
			\begin{tabular}{cccccccc}
				\toprule
				Method & Acc. & Pre. & Rec. & F1 & FAR & MDR & AUC \\
				\midrule
				AE-LOF & 51.59 & 51.53 & 50.60 & 40.90 & 8.15 & 90.66 & 50.60 \\
				AOC-IDS & 83.82 & 86.80 & 84.13 & 83.43 & 29.20 & \textbf{2.54} & 84.13 \\
				DAGMM & 79.92 & 80.88 & 79.61 & 79.62 & 10.11 & 29.72 & 76.60 \\
				DeepAID & 85.97 & 86.95 & 86.25 & 85.93 & 22.03 & 5.45 & 86.25 \\
				GDN & 86.33 & 87.03 & 86.51 & 86.24 & 18.82 & 8.16 & 86.51 \\
				AdaptiveNAD & \textbf{87.56} & \textbf{89.64} & \textbf{87.27} & \textbf{87.20} & \textbf{0.71} & 24.74 & \textbf{87.27} \\
				\bottomrule
			\end{tabular}
		\end{table}
		
		\begin{table}[t]
			\caption{Comparison with deep learning baselines on Edge-IIoTset (\%).}
			\label{tb:comparisonLatestonEdge-IIoTset}
			\centering
			\setlength\tabcolsep{1.5pt}
			\begin{tabular}{cccccccc}
				\toprule
				Method & Acc. & Pre. & Rec. & F1 & FAR & MDR & AUC \\
				\midrule
				AE-LOF & 72.82 & 77.45 & 79.08 & 72.59 & 40.15 & 1.63 & 79.06 \\
				AOC-IDS & 52.31 & 70.09 & 64.04 & 48.73 & 71.82 & \textbf{0.03} & 64.04 \\
				DAGMM & 59.43 & 69.61 & 67.57 & 58.81 & 57.35 & 7.45 & 67.56 \\
				DeepAID & 91.87 & 93.45 & 88.54 & 90.38 & 1.35 & 21.51 & 88.51 \\
				GDN & 40.32 & 30.14 & 54.74 & 33.26 & 89.43 & 1.02 & 54.74 \\
				AdaptiveNAD & \textbf{93.71} & \textbf{95.59} & \textbf{90.70} & \textbf{92.55} & \textbf{0.08} & 18.52 & \textbf{90.70} \\
				\bottomrule
			\end{tabular}
		\end{table}
		
		\begin{table}[t]
			\caption{Inference Latency Comparison (microseconds per sample).}
			\label{tab:latency_comparison}
			\centering
			\setlength\tabcolsep{1.5pt}
			\begin{tabular}{cccc}
				\toprule
				\textbf{Method} & \textbf{CIC-Darknet2020} & \textbf{NSL-KDD} & \textbf{Edge-IIoTset} \\
				\midrule
				AE-LOF & 1042.45 & 1063.63 & 1027.73 \\
				AOC-IDS & 605.29 & 606.79 & 606.83 \\
				DAGMM & 3276.29 & 3295.57 & 3332.14 \\
				DeepAID & 286.30 & 288.70 & 298.70 \\
				GDN & 1709.54 & 1708.06 & 1777.11 \\
				\textbf{AdaptiveNAD} & \textbf{247.20} & \textbf{251.97} & \textbf{245.69} \\
				\bottomrule
			\end{tabular}
		\end{table}
		
		\subsubsection{Comparison with Deep Learning Baselines}
		
		To benchmark the effectiveness of AdaptiveNAD, we compare it against five state-of-the-art unsupervised baselines: AE-LOF, AOC-IDS, DAGMM, DeepAID, and GDN. 
		The comparative results across the three datasets are detailed in Tables~\ref{tb:comparisonLatestonDarkNet}, \ref{tb:comparisonLatestonNSL-KDD}, and \ref{tb:comparisonLatestonEdge-IIoTset}, with inference latency metrics provided in Table~\ref{tab:latency_comparison}.
		
		\textbf{Superior Generalization Across Diverse Environments.} 
		AdaptiveNAD demonstrates robust generalization capabilities, achieving the highest F1-scores across all three heterogeneous benchmarks (Tables~\ref{tb:comparisonLatestonDarkNet}--\ref{tb:comparisonLatestonEdge-IIoTset}). 
		On CIC-Darknet2020, our model attains an F1-score of 76.74\%, significantly outperforming the second-best baseline (DeepAID) by a margin of 4.39\%. 
		This advantage is further amplified on Edge-IIoTset, where AdaptiveNAD achieves 92.55\% F1, surpassing DeepAID (90.38\%) and exhibiting substantial resilience compared to GDN and AOC-IDS, which suffer from catastrophic failures (F1 scores below 50\%) due to the high variability of IoT traffic. 
		Even on the traditional NSL-KDD dataset, AdaptiveNAD maintains its lead with an F1-score of 87.20\%, outperforming both GDN (86.24\%) and DeepAID (85.93\%). 
		These results confirm that our framework effectively adapts to diverse network characteristics—from encrypted tunnels to IoT protocols—without requiring architecture-specific tuning.

		\textbf{Operational Reliability: Minimizing False Alarms.} 
		Beyond aggregate accuracy, AdaptiveNAD exhibits superior operational reliability by maintaining the lowest False Alarm Rates (FAR) across all datasets. 
		High false alarm rates render NIDS impractical for real-world deployment due to alert fatigue. 
		While baselines like DeepAID and AOC-IDS achieve low missed detection rates on NSL-KDD (MDR $<$ 6\%), they incur prohibitive FARs of 22.03\% and 29.20\%, respectively. 
		In sharp contrast, AdaptiveNAD suppresses the FAR to a negligible 0.71\% on the same dataset—a reduction of over 30$\times$ compared to DeepAID. 
		Similarly, on Edge-IIoTset, while GDN suffers from a catastrophic FAR of 89.43\%, AdaptiveNAD maintains a near-zero rate of 0.08\%. 
		Although this rigorous suppression of false positives leads to a moderate trade-off in missed detections on NSL-KDD (24.74\% MDR), the result is a highly usable system that avoids overwhelming security analysts with false alerts.

		\textbf{Real-time Inference Efficiency.} 
		Table~\ref{tab:latency_comparison} confirms that AdaptiveNAD achieves the lowest inference latency, averaging approximately 245--252~$\mu$s per sample. 
		Crucially, these metrics are evaluated in strict single-sample inference mode to simulate real-time stream processing, where buffering flows for batch processing creates unacceptable security delays.
		Under this setting, our efficiency represents a 12$\times$ speedup over reconstruction-heavy baselines like DAGMM ($\sim$3300~$\mu$s) and a 6$\times$ speedup over graph-based methods like GDN ($\sim$1700~$\mu$s). 
		Even compared to DeepAID ($\sim$290~$\mu$s), our framework reduces latency by 13--18\%. 
		This speed advantage comes from the architectural difference: Random Forest executes efficient tree traversals for single predictions, whereas deep neural networks incur significant overhead from matrix operations and gradient computations in non-batched environments.

		\subsection{Ablation Study}
		\label{subsec:ablation}
		
		To disentangle the contributions of the supervised classifier integration and the adaptive thresholding mechanism, we evaluate three system variants: 
		\textbf{NAD} (the baseline VAE model without the Random Forest classifier), 
		\textbf{Fixed-NAD} (VAE integrated with Random Forest, utilizing static dual thresholds derived solely from the initial training phase), and 
		\textbf{AdaptiveNAD} (the proposed framework where dual thresholds are asynchronously updated to accommodate distribution shifts). 
		This comparative analysis isolates the specific gains in detection accuracy, false alarm reduction, and computational efficiency.
		
		\textbf{Random Forest Integration.} 
		As detailed in Tables~\ref{tab:ablation_darknet}--\ref{tab:ablation_iiotset}, Fixed-NAD yields substantial F1 gains of 39.44\% and 41.22\% over NAD on NSL-KDD and Edge-IIoTset, respectively. 
		This validates the necessity of a discriminative classifier for handling complex anomaly patterns. 
		The poor performance of NAD (e.g., 43.62\% F1 on NSL-KDD) stems from the limitation of reconstruction-based methods in distinguishing fine-grained anomalies from normal traffic, resulting in severe under-detection (MDR exceeding 87\%). 
		In contrast, Fixed-NAD leverages the discriminative decision boundaries of Random Forest, effectively mitigating this issue. 
		On CIC-Darknet2020, Fixed-NAD shows marginal improvement (0.07\% F1), suggesting that for datasets with simpler patterns, VAE's reconstruction error alone provides sufficient separability.
		
		\textbf{Adaptive Threshold Mechanism.} 
		The comparison between Fixed-NAD and AdaptiveNAD highlights the importance of dynamic adaptation. 
		AdaptiveNAD outperforms Fixed-NAD by 9.18\%, 4.14\%, and 5.45\% F1 on CIC-Darknet2020, NSL-KDD, and Edge-IIoTset, respectively. 
		These results confirm that periodically updating decision thresholds enables the system to track evolving data distributions, a capability lacking in Fixed-NAD's static boundaries. 
		The improvement is most pronounced on CIC-Darknet2020 (9.18\%), where extreme class imbalance (67:1) renders static thresholds ineffective against shifting normal-abnormal boundaries. 
		The adaptive mechanism continuously recalibrates these boundaries based on accumulated samples, thereby sustaining detection performance during deployment.
		
		\textbf{Latency Analysis.} 
		Table~\ref{tab:ablation_latency} compares the computational overhead. Notably, we evaluate these metrics using strictly unbatched, single-sample inference to reflect real-time streaming, since the temporal overhead of buffering flows for batch processing severely compromises immediate threat response.
		Both Fixed-NAD and AdaptiveNAD achieve a 3.6--3.8$\times$ speedup over NAD, reducing per-sample latency from approximately 900~$\mu$s to 250~$\mu$s. 
		This acceleration occurs because the Random Forest classifier, which relies on efficient tree traversal, replaces the computationally expensive LSTM forward propagation and reconstruction steps required by the VAE-only baseline for anomaly scoring. 
		Notably, the latency difference between Fixed-NAD and AdaptiveNAD is negligible (less than 1.2~$\mu$s). 
		This confirms that the threshold updates occur asynchronously during the model refinement phase without blocking the real-time inference pipeline.
		
		\textbf{Trade-off Analysis.} 
		The ablation results reveal a fundamental trade-off between FAR and MDR. 
		NAD suffers from severe under-detection on complex datasets (MDR exceeding 87\% on NSL-KDD and Edge-IIoTset). 
		Fixed-NAD effectively reduces MDR but exhibits instability: it either spikes the FAR (e.g., 24.24\% on NSL-KDD) or regresses in detection rate (e.g., MDR rising to 24.64\% on CIC-Darknet2020) due to rigid decision boundaries. 
		AdaptiveNAD resolves this dilemma by dynamically fine-tuning thresholds. 
		It converges to a balanced operating point that maximizes the F1-score: on NSL-KDD, it sacrifices a portion of the detection rate (MDR 24.74\%) to drastically cut false alarms (FAR 0.71\%), while on Edge-IIoTset, it optimizes both metrics simultaneously. 
		This demonstrates that the adaptive component serves as a crucial regularizer, calibrating system sensitivity to the specific characteristics of each network environment.
		
		\begin{table}[t]
			\caption{Ablation study on CIC-Darknet2020 (\%).}
			\label{tab:ablation_darknet}
			\centering
			\setlength\tabcolsep{1.5pt}
			\begin{tabular}{cccccccc}
				\toprule
				Method & Acc. & Pre. & Rec. & F1 & FAR & MDR & AUC \\
				\midrule
				NAD & 95.92 & 66.81 & 84.03 & 67.49 & 3.87 & \textbf{9.19} & 65.90 \\
				Fixed-NAD & 96.87 & 69.60 & 85.41 & 67.56 & 2.31 & 24.64 & 76.88 \\
				AdaptiveNAD & \textbf{98.24} & \textbf{75.51} & \textbf{92.12} & \textbf{76.74} & \textbf{1.33} & 12.21 & \textbf{88.78} \\
				\bottomrule
			\end{tabular}
		\end{table}
		
		\begin{table}[t]
			\caption{Ablation study on NSL-KDD (\%).}
			\label{tab:ablation_nslkdd}
			\centering
			\setlength\tabcolsep{1.5pt}
			\begin{tabular}{cccccccc}
				\toprule
				Method & Acc. & Pre. & Rec. & F1 & FAR & MDR & AUC \\
				\midrule
				NAD & 57.38 & 62.87 & 56.00 & 43.62 & \textbf{0.41} & 87.60 & 55.26 \\
				Fixed-NAD & 83.24 & 83.99 & 83.49 & 83.06 & 24.24 & \textbf{8.77} & 53.64 \\
				AdaptiveNAD & \textbf{87.56} & \textbf{89.64} & \textbf{87.27} & \textbf{87.20} & 0.71 & 24.74 & \textbf{84.78} \\
				\bottomrule
			\end{tabular}
		\end{table}
		
		\begin{table}[t]
			\caption{Ablation study on Edge-IIoTset (\%).}
			\label{tab:ablation_iiotset}
			\centering
			\setlength\tabcolsep{1.5pt}
			\begin{tabular}{cccccccc}
				\toprule
				Method & Acc. & Pre. & Rec. & F1 & FAR & MDR & AUC \\
				\midrule
				NAD & 68.42 & 83.66 & 53.16 & 45.88 & 0.10 & 93.64 & 53.00 \\
				Fixed-NAD & 89.55 & 93.11 & 84.51 & 87.10 & \textbf{0.07} & 30.90 & 84.48 \\
				AdaptiveNAD & \textbf{93.71} & \textbf{95.59} & \textbf{90.70} & \textbf{92.55} & 0.08 & \textbf{18.52} & \textbf{90.70} \\
				\bottomrule
			\end{tabular}
		\end{table}
		
		\begin{table}[t]
			\centering
			\caption{\textcolor{black}{Online Inference Latency (microseconds per sample)}}
			\label{tab:ablation_latency}
			\setlength\tabcolsep{1.5pt}
			\begin{tabular}{cccc}
				\toprule
				\textbf{Method} & \textbf{CIC-Darknet2020} & \textbf{NSL-KDD} & \textbf{Edge-IIoTset} \\
				\midrule
				NAD & 918.51 & 961.02 & 892.37 \\
				Fix-NAD & 247.15 & 251.10 & 246.81 \\
				AdaptiveNAD & 247.20 & 251.97 & 245.69 \\
				\midrule
				Speedup & \textbf{3.72}$\times$ & \textbf{3.83}$\times$ & \textbf{3.63}$\times$ \\
				\bottomrule
			\end{tabular}
		\end{table}

		\section{Related Work} \label{sec:RelatedWork}
		In this section, we briefly discuss existing literature from the following three aspects, unsupervised anomaly detection models and online unsupervised anomaly detection approaches.
		
		\textbf{Unsupervised anomaly detection models:} 
		The "Zero-positive" setting, where some normal samples are known by a model, is one of the most common assumptions for unsupervised anomaly detection. Existing literature from this setting can be roughly grouped into three categories, i.e., reconstruction-based, prediction-based, and contrastive-based approaches.
		
		Reconstruction-based methods use the encoder and decoder to reconstruct the input and those with large reconstruction errors are recognized as anomalies. 
		For instance, 
		\cite{audibert2020usad,tuli2023deepft} utilizes multiple encoders and decoders to enhance the representation ability of the model. 
		Moreover, in \cite{abdulaal2021practical}, spectral analysis is employed within a pre-trained autoencoder to synchronize phase information across complex asynchronous multivariate time series data for anomaly detection by minimizing quantile reconstruction losses. \cite{li2021multivariate} introduces an unsupervised approach that models normal data through a hierarchical Variational AutoEncoder with two latent variables for capturing inter-metric and temporal dependencies in time series data. Reconstruction-based methods do not demand any labeled anomalies and are well-suited for non-time series data.
		For time-series data, prediction-based and contrastive-based methods are usually utilized, whereas prediction-based approaches usually use the past series to forecast future windows or specific time stamps, while contrastive-based methods learn data representations through comparisons between positive and negative samples. Given an intuitive assumption, which is neighboring time windows or timestamps in the raw time series share high similarity, thus are considered as positive samples, and those from distant time windows or timestamps are treated as negative samples \cite{oord2018representation}. Similar to reconstruction-based methods, prediction-based approaches can predict those from normal time series with higher probability while tending to wrongly predict abnormal data. \cite{du2019lifelong} considers the scenarios in which operators can manually examine a few suspicious data samples to provide their labels. It improves anomaly detection performance by maximizing the prediction loss to unlearn the reported abnormal samples. \cite{deng2021graph} improves anomaly detection performance by considering relationships between different sensors. 
		\cite{tuli2023deepft} generates self-supervised anomaly labels by reconstructing an input time-series window using a system state decoder, which provides an accurate prediction of the subsequent state. Anomalous labels are assigned to instances where the predicted state exceeds a dynamically determined threshold, while benign labels are assigned to instances where the predicted state falls below this threshold. The threshold is dynamically adjusted based on the extreme values observed in the input time-series data. \cite{han2021deepaid} employs a typical prediction-based anomaly detection method that trains a multivariate LSTM by minimizing the prediction loss of the model's performance of predicting the current time series using previous time series data.
		As a contrastive-based method, \cite{yang2023dcdetector} proposes a contrastive learning-based method that learns representations that distinguish normal time series points from anomalies by maximizing the similarity between two views (patch-wise and in-patch representations) of normal data and enlarging the discrepancy for anomalies. \cite{tonekaboni2021unsupervised} proposed a self-supervised framework tailored for non-stationary time series based on positive-unlabeled (PU) learning that treats negative samples as unknown samples and assigns weights to these samples to handle the problem of sampling bias incurred by non-stationary characteristics of most time series. 
		
		Except for these, some research uses hybrid approaches that optimize a combination of reconstruction loss, prediction loss, or contrastive loss. For instance, \cite{zong2018deep} combines a deep autoencoder with a Gaussian Mixture Model (GMM) and detects anomalies by jointly minimizing a combination of reconstruction loss and energy function. It outperforms traditional clustering methods (e.g., k-means) by effectively handling high-dimensional data through integrated dimensionality reduction and clustering in a single step. \cite{han2022learning} develops a Fused Sparse Autoencoder and Graph Net, which minimizes reconstruction and prediction losses while modeling the relationships between features in multivariate time series. \cite{liu2024timesurl} introduces a self-supervised learning framework that learns time series representations by jointly optimizing reconstruction loss and contrastive loss for both segment and instance levels. It also enhances contrastive learning's performance by injecting more hard negative samples.
		
		
		Although many previous works have validated the flexibility and effectiveness of "zero-positive" anomaly detection through reconstruction-based, prediction-based, or contrastive learning-based approaches, 
		the effective incremental update of the initial model to changing new patterns without any label supervision. 
		Besides, most existing work 
		struggles with high false alarm rates or long latency, while those features are vital for many practical applications such as smart healthcare, manufacturing, or autonomous driving \cite{esteva2019guide,wang2018deep,levinson2011towards}. 
		
		\textbf{Online unsupervised anomaly detection approaches:} Recently, there is only limited research considering online unsupervised anomaly detection for improving ADSs' practicality, such as reducing the false alarm rate \cite{baldini2022online,wang2022distributed,chen2021daemon,du2019lifelong,AOCIDS,odiathevar2021online,zhang2023real}, detecting concept drift \cite{du2019lifelong,AOCIDS,odiathevar2021online,zhang2023real}.
		However, most of them like \cite{baldini2022online,wang2022distributed,chen2021daemon} consider \textquotedblleft online\textquotedblright~as the model being trained offline and utilized for real-time inferences without model updating. Other research, like \cite{du2019lifelong,AOCIDS,odiathevar2021online,zhang2023real} proposes real online unsupervised anomaly detection methods that incrementally update the model to remain synchronized with dynamic environments. For example, \cite{AOCIDS} introduces a pseudo-label generator that generates data labels in a streaming fashion to re-train the model. \cite{odiathevar2021online} adapts the threshold when a batch of input data varies from normal data. 
		
		\section{Conclusions} \label{sec:Conclusions}
		In this paper, we propose Adaptive NAD, a general framework to improve online unsupervised anomaly detection in real-world applications in security domains. In Adaptive NAD, a two-layer anomaly detection strategy (ITAD-S) is designed to generate reliable high-confidence pseudo-labels and provide generalization. Besides, \textcolor{black}{an online training scheme} has been developed to update both the unsupervised deep learning model and the loss distributions using a well-designed threshold calculation technique that offers low false positives. By applying and evaluating Adaptive NAD over two classic anomaly detection datasets, we demonstrate that Adaptive NAD can provide low latency and low false positive rates in online anomaly detection. As part of future work, more models will be attempted to select as first or second layer of Adaptive NAD for optimizing the performance of the proposed framework. Additionally, we will investigate more techniques to automatically configure the hyper-parameters and integrate network events from different real-world systems and use these data to further evaluate the robustness of ADSs built using the Adaptive NAD framework.
		

		\section*{Acknowledgment}
		This work was supported in part supported by the National Natural Science Foundation of China (Grant No. 62406215).
		
		\bibliographystyle{elsarticle-num} 
		\bibliography{main}

@inproceedings{tavallaee2009detailed,
  title={A detailed analysis of the KDD CUP 99 data set},
  author={Tavallaee, Mahbod and Bagheri, Ebrahim and Lu, Wei and Ghorbani, Ali A},
  booktitle={2009 IEEE symposium on computational intelligence for security and defense applications},
  pages={1--6},
  year={2009},
  organization={Ieee}
}

@techreport{2,
  author = {},
  title = {},
  year = {2024},
  institution = {U.S. Securities and Exchange Comission},
  url = {https://www.sec.gov/Archives/edgar/data/731766/000073176624000146/a2024q1exhibit991.html},
  note = {Accessed: 2025-2-13}
}

@techreport{1,
  author = {},
  title = {},
  year = {2024},
  institution = {U.S. Securities and Exchange Comission},
  url = {https://www.sec.gov/Archives/edgar/data/731766/000073176624000146/a2024q1exhibit991.html},
  note = {Accessed: 2025-2-13}
}

@data{mbc1-1h68-22,
doi = {10.21227/mbc1-1h68},
url = {https://dx.doi.org/10.21227/mbc1-1h68},
author = {Ferrag, Mohamed Amine and Friha, Othmane and Hamouda, Djallel and Maglaras, Leandros and Janicke, Helge},
publisher = {IEEE Dataport},
title = {Edge-IIoTset: A New Comprehensive Realistic Cyber Security Dataset of IoT and IIoT Applications: Centralized and Federated Learning},
year = {2022} }

@article{wang2022distributed,
  title={Distributed online anomaly detection for virtualized network slicing environment},
  author={Wang, Weili and Liang, Chengchao and Chen, Qianbin and Tang, Lun and Yanikomeroglu, Halim and Liu, Tong},
  journal={IEEE Transactions on Vehicular Technology},
  volume={71},
  number={11},
  pages={12235--12249},
  year={2022},
  publisher={IEEE}
}

@article{baldini2022online,
  title={Online Distributed Denial of Service (DDoS) intrusion detection based on adaptive sliding window and morphological fractal dimension},
  author={Baldini, Gianmarco and Amerini, Irene},
  journal={Computer Networks},
  volume={210},
  pages={108923},
  year={2022},
  publisher={Elsevier}
}

@inproceedings{chen2021daemon,
  title={Daemon: Unsupervised anomaly detection and interpretation for multivariate time series},
  author={Chen, Xuanhao and Deng, Liwei and Huang, Feiteng and Zhang, Chengwei and Zhang, Zongquan and Zhao, Yan and Zheng, Kai},
  booktitle={2021 IEEE 37th International Conference on Data Engineering (ICDE)},
  pages={2225--2230},
  year={2021},
  organization={IEEE}
}

@article{zhang2023real,
  title={Real-time malicious traffic detection with online isolation forest over sd-wan},
  author={Zhang, Pei and He, Fangzhou and Zhang, Han and Hu, Jiankun and Huang, Xiaohong and Wang, Jilong and Yin, Xia and Zhu, Huahong and Li, Yahui},
  journal={IEEE Transactions on Information Forensics and Security},
  volume={18},
  pages={2076--2090},
  year={2023},
  publisher={IEEE}
}

@inproceedings{li2021multivariate,
  title={Multivariate time series anomaly detection and interpretation using hierarchical inter-metric and temporal embedding},
  author={Li, Zhihan and Zhao, Youjian and Han, Jiaqi and Su, Ya and Jiao, Rui and Wen, Xidao and Pei, Dan},
  booktitle={Proceedings of the 27th ACM SIGKDD conference on knowledge discovery \& data mining},
  pages={3220--3230},
  year={2021}
}

@inproceedings{abdulaal2021practical,
  title={Practical approach to asynchronous multivariate time series anomaly detection and localization},
  author={Abdulaal, Ahmed and Liu, Zhuanghua and Lancewicki, Tomer},
  booktitle={Proceedings of the 27th ACM SIGKDD conference on knowledge discovery \& data mining},
  pages={2485--2494},
  year={2021}
}

@inproceedings{han2022learning,
  title={Learning sparse latent graph representations for anomaly detection in multivariate time series},
  author={Han, Siho and Woo, Simon S},
  booktitle={Proceedings of the 28th ACM SIGKDD Conference on knowledge discovery and data mining},
  pages={2977--2986},
  year={2022}
}

@inproceedings{audibert2020usad,
  title={Usad: Unsupervised anomaly detection on multivariate time series},
  author={Audibert, Julien and Michiardi, Pietro and Guyard, Fr{\'e}d{\'e}ric and Marti, S{\'e}bastien and Zuluaga, Maria A},
  booktitle={Proceedings of the 26th ACM SIGKDD international conference on knowledge discovery \& data mining},
  pages={3395--3404},
  year={2020}
}

@article{oord2018representation,
  title={Representation learning with contrastive predictive coding},
  author={Oord, Aaron van den and Li, Yazhe and Vinyals, Oriol},
  journal={arXiv preprint arXiv:1807.03748},
  year={2018}
}

@inproceedings{
tonekaboni2021unsupervised,
title={Unsupervised Representation Learning for Time Series with Temporal Neighborhood Coding},
author={Sana Tonekaboni and Danny Eytan and Anna Goldenberg},
booktitle={International Conference on Learning Representations},
year={2021},
url={https://openreview.net/forum?id=8qDwejCuCN}
}

@inproceedings{liu2024timesurl,
  title={Timesurl: Self-supervised contrastive learning for universal time series representation learning},
  author={Liu, Jiexi and Chen, Songcan},
  booktitle={Proceedings of the AAAI Conference on Artificial Intelligence},
  volume={38},
  number={12},
  pages={13918--13926},
  year={2024}
}

@inproceedings{yang2023dcdetector,
  title={Dcdetector: Dual attention contrastive representation learning for time series anomaly detection},
  author={Yang, Yiyuan and Zhang, Chaoli and Zhou, Tian and Wen, Qingsong and Sun, Liang},
  booktitle={Proceedings of the 29th ACM SIGKDD Conference on Knowledge Discovery and Data Mining},
  pages={3033--3045},
  year={2023}
}

@inproceedings{levinson2011towards,
  title={Towards fully autonomous driving: Systems and algorithms},
  author={Levinson, Jesse and Askeland, Jake and Becker, Jan and Dolson, Jennifer and Held, David and Kammel, Soeren and Kolter, J Zico and Langer, Dirk and Pink, Oliver and Pratt, Vaughan and others},
  booktitle={2011 IEEE intelligent vehicles symposium (IV)},
  pages={163--168},
  year={2011},
  organization={IEEE}
}

@article{wang2018deep,
  title={Deep learning for smart manufacturing: Methods and applications},
  author={Wang, Jinjiang and Ma, Yulin and Zhang, Laibin and Gao, Robert X and Wu, Dazhong},
  journal={Journal of manufacturing systems},
  volume={48},
  pages={144--156},
  year={2018},
  publisher={Elsevier}
}

@article{esteva2019guide,
  title={A guide to deep learning in healthcare},
  author={Esteva, Andre and Robicquet, Alexandre and Ramsundar, Bharath and Kuleshov, Volodymyr and DePristo, Mark and Chou, Katherine and Cui, Claire and Corrado, Greg and Thrun, Sebastian and Dean, Jeff},
  journal={Nature medicine},
  volume={25},
  number={1},
  pages={24--29},
  year={2019},
  publisher={Nature Publishing Group US New York}
}

@inproceedings{du2019lifelong,
  title={Lifelong anomaly detection through unlearning},
  author={Du, Min and Chen, Zhi and Liu, Chang and Oak, Rajvardhan and Song, Dawn},
  booktitle={Proceedings of the 2019 ACM SIGSAC conference on computer and communications security},
  pages={1283--1297},
  year={2019}
}

@article{kanungo2002efficient,
  title={An efficient k-means clustering algorithm: Analysis and implementation},
  author={Kanungo, Tapas and Mount, David M and Netanyahu, Nathan S and Piatko, Christine D and Silverman, Ruth and Wu, Angela Y},
  journal={IEEE transactions on pattern analysis and machine intelligence},
  volume={24},
  number={7},
  pages={881--892},
  year={2002},
  publisher={IEEE}
}

@inproceedings{deng2021graph,
  title={Graph neural network-based anomaly detection in multivariate time series},
  author={Deng, Ailin and Hooi, Bryan},
  booktitle={Proceedings of the AAAI conference on artificial intelligence},
  volume={35},
  number={5},
  pages={4027--4035},
  year={2021}
}

@inproceedings{tuli2023deepft,
  title={Deepft: Fault-tolerant edge computing using a self-supervised deep surrogate model},
  author={Tuli, Shreshth and Casale, Giuliano and Cherkasova, Ludmila and Jennings, Nicholas R},
  booktitle={IEEE INFOCOM 2023-IEEE Conference on Computer Communications},
  pages={1--10},
  year={2023},
  organization={IEEE}
}

@inproceedings{AOCIDS,
  title={AOC-IDS: Autonomous Online Framework with Contrastive Learning for Intrusion Detection},
  author={Zhang, Xinchen and Zhao, Running and Jiang, Zhihan and Sun, Zhicong and Ding, Yulong and Ngai, Edith CH and Yang, Shuang-Hua},
  booktitle={IEEE INFOCOM 2024-IEEE Conference on Computer Communications},
  pages={1--10},
  year={2024},
  organization={IEEE}
}

@inproceedings{han2021deepaid,
  title={Deepaid: Interpreting and improving deep learning-based anomaly detection in security applications},
  author={Han, Dongqi and Wang, Zhiliang and Chen, Wenqi and Zhong, Ying and Wang, Su and Zhang, Han and Yang, Jiahai and Shi, Xingang and Yin, Xia},
  booktitle={Proceedings of the 2021 ACM SIGSAC Conference on Computer and Communications Security},
  pages={3197--3217},
  year={2021}
}

@article{rigatti2017random,
  title={Random forest},
  author={Rigatti, Steven J},
  journal={Journal of Insurance Medicine},
  volume={47},
  number={1},
  pages={31--39},
  year={2017},
  publisher={American Academy of Insurance Medicine 1700 Magnavox Way, Fort Wayne, IN 46804}
}

@inproceedings{habibi2020didarknet,
  title={Didarknet: A contemporary approach to detect and characterize the darknet traffic using deep image learning},
  author={Habibi Lashkari, Arash and Kaur, Gurdip and Rahali, Abir},
  booktitle={Proceedings of the 2020 10th International Conference on Communication and Network Security},
  pages={1--13},
  year={2020}
}

@article{chen2024network,
  title={Network anomaly detection via similarity-aware ensemble learning with ADSim},
  author={Chen, Wenqi and Wang, Zhiliang and Chang, Liyuan and Wang, Kai and Zhong, Ying and Han, Dongqi and Duan, Chenxin and Yin, Xia and Yang, Jiahai and Shi, Xingang},
  journal={Computer Networks},
  volume={247},
  pages={110423},
  year={2024},
  publisher={Elsevier}
}

@article{bovenzi2023network,
  title={Network anomaly detection methods in IoT environments via deep learning: A Fair comparison of performance and robustness},
  author={Bovenzi, Giampaolo and Aceto, Giuseppe and Ciuonzo, Domenico and Montieri, Antonio and Persico, Valerio and Pescap{\'e}, Antonio},
  journal={Computers \& Security},
  volume={128},
  pages={103167},
  year={2023},
  publisher={Elsevier}
}

@article{saba2022anomaly,
  title={Anomaly-based intrusion detection system for IoT networks through deep learning model},
  author={Saba, Tanzila and Rehman, Amjad and Sadad, Tariq and Kolivand, Hoshang and Bahaj, Saeed Ali},
  journal={Computers and Electrical Engineering},
  volume={99},
  pages={107810},
  year={2022},
  publisher={Elsevier}
}

@inproceedings{chen2019unsupervised,
  title={Unsupervised Anomaly Detection for Intricate KPIs via Adversarial Training of VAE},
  author={Chen, Wenxiao and Xu, Haowen and Li, Zeyan and Peiy, Dan and Chen, Jie and Qiao, Honglin and Feng, Yang and Wang, Zhaogang},
  booktitle={IEEE INFOCOM 2019-IEEE Conference on Computer Communications},
  pages={1891--1899},
  year={2019},
  organization={IEEE}
}

@article{zong2018deep,
  title={Deep autoencoding gaussian mixture model for unsupervised anomaly detection},
  author={Zong, Bo and Song, Qi and Min, Martin Renqiang and Cheng, Wei and Lumezanu, Cristian and Cho, Daeki and Chen, Haifeng},
  journal={ICLR},
  year={2018}
}

@article{aguayo2018novelty,
  title={Novelty Detection in Time Series Using Self-Organizing Neural Networks: A Comprehensive Evaluation},
  author={Aguayo, Leonardo and Barreto, Guilherme A},
  journal={Neural Processing Letters},
  volume={47},
  number={2},
  pages={717--744},
  year={2018},
  publisher={Springer}
}

@article{kingma2013auto,
  title={Auto-encoding variational bayes},
  author={Kingma, Diederik P and Welling, Max},
  journal={arXiv preprint arXiv:1312.6114},
  year={2013}
}

@article{doersch2016tutorial,
  title={Tutorial on variational autoencoders},
  author={Doersch, Carl},
  journal={arXiv preprint arXiv:1606.05908},
  year={2016}
}

@article{odiathevar2021online,
  title={An online offline framework for anomaly scoring and detecting new traffic in network streams},
  author={Odiathevar, Murugaraj and Seah, Winston KG and Frean, Marcus and Valera, Alvin},
  journal={IEEE Transactions on Knowledge and Data Engineering},
  volume={34},
  number={11},
  pages={5166--5181},
  year={2021},
  publisher={IEEE}
}

@inproceedings{tan2011fast,
  title={Fast anomaly detection for streaming data},
  author={Tan, Swee Chuan and Ting, Kai Ming and Liu, Tony Fei},
  booktitle={IJCAI proceedings-international joint conference on artificial intelligence},
  volume={22},
  number={1},
  pages={1511},
  year={2011}
}

@article{montiel2021river,
  title={River: machine learning for streaming data in python},
  author={Montiel, Jacob and Halford, Max and Mastelini, Saulo Martiello and Bolmier, Geoffrey and Sourty, Raphael and Vaysse, Robin and Zouitine, Adil and Gomes, Heitor Murilo and Read, Jesse and Abdessalem, Talel and others},
  journal={Journal of Machine Learning Research},
  volume={22},
  number={110},
  pages={1--8},
  year={2021}
}

		\appendices
		\section{Appendices}
		To show the generality of the proposed threshold calculation technique, we also evaluate it on the CIC-DoHBrw-2020 dataset. Fig. \ref{fi:lossDistrubutinoDoHBrw} is the empirical distribution of the losses obtained from the proposed Adaptive NAD against the best-fitting theoretical distributions. From Fig. \ref{fi:lossDistrubutinoDoHBrw}, we can see that the norm and logistic distributions are the best-fitting theoretical distributions of the normal and abnormal losses.
		
		Here, MLE and Method of Moments (MoM) are utilized to estimate the parameters of the normal and logistic distributions, respectively. 
		The parameters of the normal distribution are calculated as follows:
		\begin{theorem}
			Let $X \sim \text{normal}(\mu, \sigma^2)$. The maximum likelihood estimators under $n$ observations of the parameters are:
		\end{theorem}
		\begin{itemize}
			\item[(i)] \(
			\hat{\mu} = \frac{\sum_{j=1}^n x_j}{n},
			\)
			\item[(ii)] \(
			\hat{\sigma^2} = \frac{\sum_{j=1}^n (x_j - \hat{\mu})}{n}
			\)
		\end{itemize}
		
		\begin{proof}\renewcommand{\qedsymbol}{}
			\begin{equation}
				\begin{split}
					L(\mu,\sigma^2|X) 
					&= \prod_{j = 1}^{n}[f(x_j|\mu,\sigma^2)], \\
					&= \prod_{j = 1}^{n}\frac{1}{\sqrt{2\pi}\sigma}e^{-\frac{(x_j-\mu)^2}{2\sigma^2}}, \\
					&={(2\pi\sigma^2)}^{-n/2}e^{-\frac{(x_j-\mu)^2}{2\sigma^2}},\\
					&={-\frac{n}{2}ln(2\pi\sigma^2) -\sum_{j = 1}^{n}{\frac{(x_j-\mu)^2}{2\sigma^2}}}. \nonumber
				\end{split}
			\end{equation}
			Let:
			\begin{equation}
				\frac{\partial L}{\partial \mu} = \frac{1}{\sigma^2} \sum_{j=1}^n (x_j-\mu)=0, \nonumber
			\end{equation}
			we have $\hat{\mu}$ as:
			\begin{equation}
				\hat{\mu}=\frac{\sum_{j=1}^n x_j}{n}. \nonumber
			\end{equation}
			Let
			\begin{equation}
				\frac{\partial L}{\partial \sigma^2} = -\frac{n}{2} \cdot \frac{1}{\sigma^2}+\frac{1}{2\sigma^4} \sum_{j=1}^n (x_j-\mu)^2, \nonumber
			\end{equation}
			we have $\hat{\sigma^2}$ as:
			\begin{equation}
				\hat{\sigma^2}=\frac{1}{n}\sum_{j=1}^n(x_j-\frac{\sum_{j=1}^n x_j}{n})^2.\nonumber
			\end{equation}
		\end{proof}
		
		The parameter estimation of the logistic distribution is below:
		\begin{theorem}
			Let $X \sim \text{logistic}(\mu, \gamma)$. The maximum likelihood estimators under $n$ observations of the parameters are:
		\end{theorem}
		\begin{itemize}
			\item[(i)] \(
			\hat{\mu} = \frac{1}{n}\sum_{j=1}^n x_j,
			\)
			\item[(ii)] \(
			\hat{\gamma} = \frac{1}{\pi} \sqrt{\frac{3}{n}\cdot \sum_{j=1}^n (x_j-\frac{\sum_{j=1}^n x_j}{n})^2}.
			\)
		\end{itemize}
		
		\begin{proof}\renewcommand{\qedsymbol}{}
			Because that the expectation value of the logistic distribution is $\mu$ and its variance is $\frac{\pi^2}{3}\gamma^2$, then
			\begin{equation}
				\mu=\overline{x},\nonumber
			\end{equation}
			which means:
			\begin{equation}
				\hat{\mu} = \frac{1}{n}\sum_{j=1}^n x_j. \nonumber
			\end{equation}
			Similarly,
			\begin{equation}
				\frac{\pi^2}{3}\gamma^2 = \frac{1}{n}\sum_{j=1}^n(x_j - \mu)^2,\nonumber 
			\end{equation}
			\begin{equation}
				\hat{\gamma}=\frac{1}{\pi} \sqrt{\frac{3}{n}\cdot \sum_{j=1}^n (x_j-\frac{\sum_{j=1}^n x_j}{n})^2}.\nonumber 
			\end{equation}
		\end{proof}

		The threshold $T_1$ and $T_1$ for the CIC-DoHBrw-2020 dataset is obtained by Proposition \ref{prop:threshold2}. 
		\begin{prop}
			\label{prop:threshold2}
			Let $T_1$ and $T_2$ be the threshold for network traffic, and $p_1$ and $p_2$ are the percentiles for normal and abnormal loss distributions related to $T_1$ and $T_2$. If the normal loss follows a normal distribution with parameter $\mu_1$ and $\sigma$ while the abnormal loss follows a logistic distribution with parameter $\mu_2$ and $\gamma$, then $T_1$ and $T_2$ are calculated as follows with a ($1-\alpha$) confidence level ($0<\alpha<1$):
			\begin{equation}
				T_1 = \Phi_1^{-1}(p_1)\sigma+\mu_1.\nonumber
			\end{equation}
			\begin{equation}
				T_2 = \Phi_2^{-1}(1-p_2)\gamma+\mu_2.\nonumber
			\end{equation}
		\end{prop}
		
		\begin{proof}\renewcommand{\qedsymbol}{}
			Let $X_1$ and $X_2$ be the loss of normal observations and loss of abnormal observation with cumulative distribution functions $F_1$ and $F_2$ respectively formulated as follows:
			\begin{equation}
				F_1(T_1)=P(X_1\leq T_1)= p_1,\nonumber
			\end{equation}
			\begin{equation}
				F_2(T_2)=P(X_2\geq T_2)= p_2.\nonumber
			\end{equation}
			
			By assumption, $F_{1}(T_1)$ follows a normal distribution and $F_{2}(T_2)$ follows a logistic distribution that are defined as follows:
			\begin{align}
				F_{1}(T_1)&=P(X_1\leq T_1)\nonumber\\
				&=\int_{0}^{T_1}f_1(x_1)dx_1
				=\Phi_1\left(\frac{T_1-\mu_1}{\sigma}\right),\nonumber
			\end{align}
			\begin{align}
				F_{2}(T_2)&=1-P(X_2\leq T_2)\nonumber\\&=1-\int_{0}^{T_2}f_2(x_2)dx_2=1-\Phi_2\left(\frac{T_2-\mu_2}{\gamma}\right),\nonumber
			\end{align}
			where $\Phi_1$ and $\Phi_2$ are the cumulative distribution functions of the normal and logistic distributions. $f_1(x_1)$ and $f_2(x_2)$ are the probability density function of the normal and logistic distributions, which are calculated by:
			\begin{align}
				f_1(x_1)= \frac{1}{\sqrt{2\pi}\sigma}exp(-\frac{(x-\mu)^2}{2\sigma^2}),\\ \nonumber
				f_2(x_2)=\frac{e^{-\frac{x-u}{\gamma}}}{\gamma(1+e^{-\frac{x-u}{\gamma}})^2}.\nonumber
			\end{align}
			Based on the above equations, the threshold $T_{1}$ and $T_{2}$ are calculated as follows:
			\begin{equation}
				p_1 = \Phi_1\left(\frac{T_1-\mu_1}{\sigma}\right),\nonumber
			\end{equation}
			\begin{equation}
				p_2 = 1-\Phi_2\left(\frac{T_2-\mu_2}{\gamma}\right),\nonumber
			\end{equation}
			\begin{equation}
				\Phi_1^{-1}(p_1) =\frac{T_1-\mu_1}{\sigma},\nonumber
			\end{equation}
			\begin{equation}
				\Phi_2^{-1}(1-p_2) =\frac{T_2-\mu_2}{\gamma},\nonumber
			\end{equation}
			\begin{equation}
				T_1 = \Phi_1^{-1}(p_1)\sigma+\mu_1,\nonumber
			\end{equation}
			\begin{equation}
				T_2 = \Phi_2^{-1}(1-p_2)\gamma+\mu_2.\nonumber
			\end{equation}
		\end{proof}

	\begin{figure}[t]
		\centering
		\begin{minipage}{.24\textwidth}
			\scalebox{1.}{\includegraphics[width=1.\linewidth]{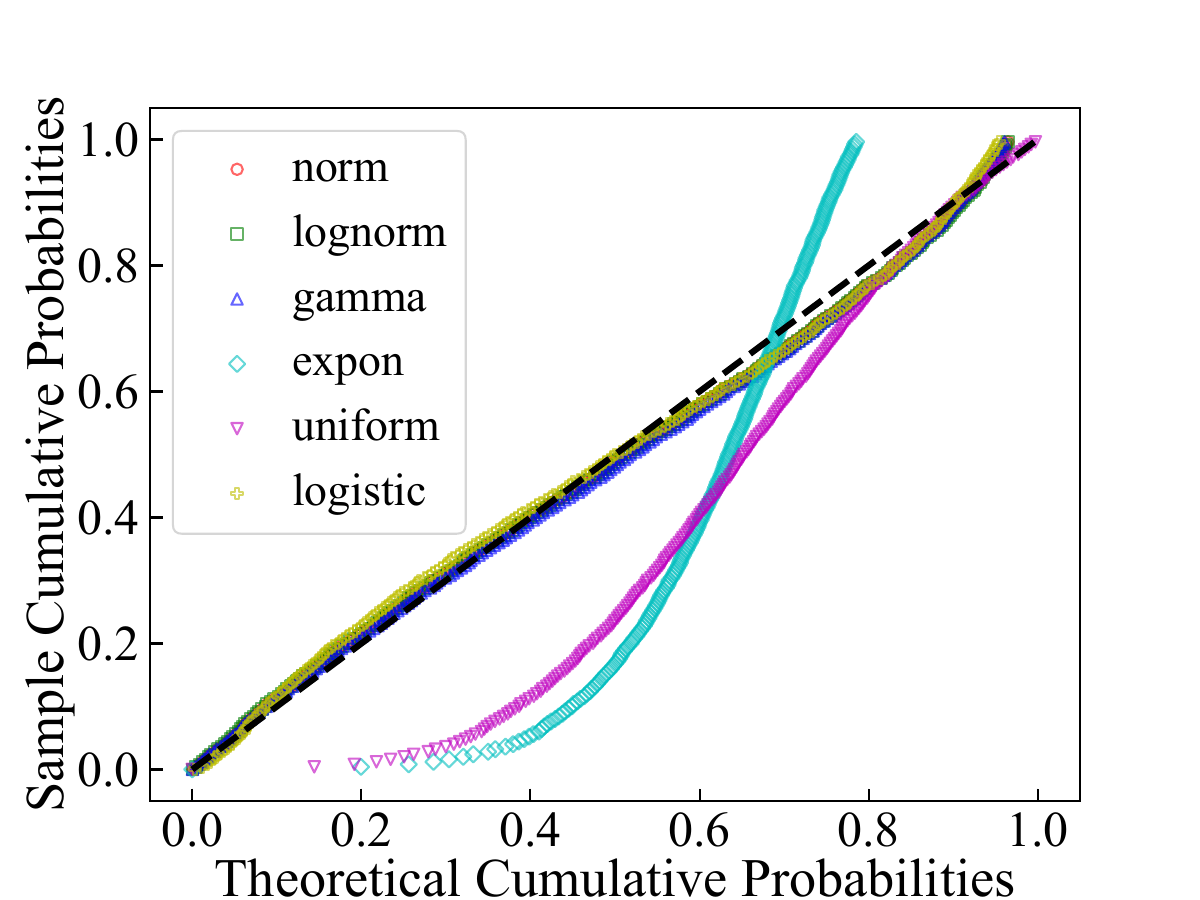}} \\
			\vspace{-1mm}
			(a) Normal Distribution.
		\end{minipage}%
		\begin{minipage}{.24\textwidth}
			\scalebox{1}{\includegraphics[width=1\linewidth]{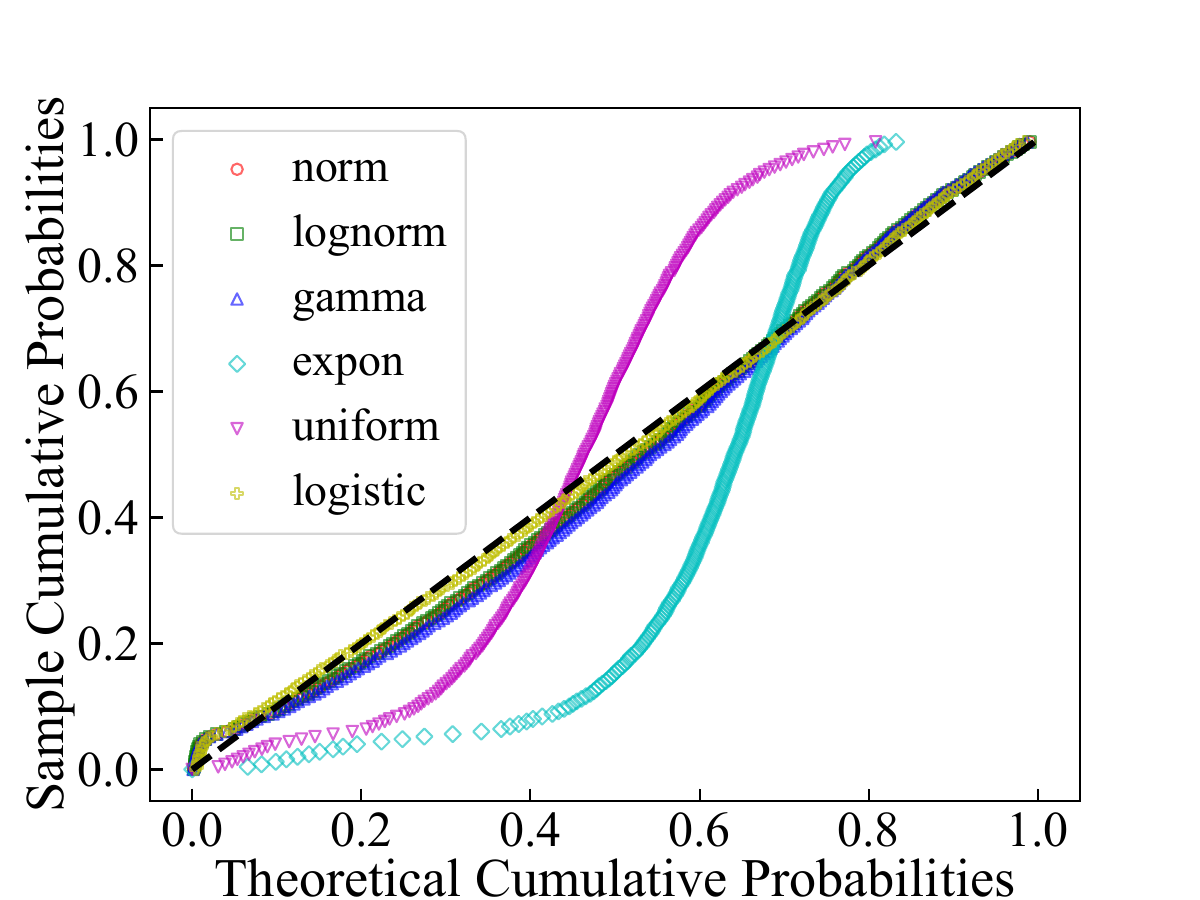}} \\
			\vspace{-1mm}
			(b) Abnormal Distribution.
		\end{minipage}%
		\caption{Loss distributions of normal and abnormal network traffic for CIC-DoHBrw-2020.}
		\label{fi:lossDistrubutinoDoHBrw}
		\vspace{-2mm}
	\end{figure}
	
	\begin{figure}[t]
		\centering
		\begin{minipage}{.24\textwidth}
			\centering
			\scalebox{.23}{\includegraphics{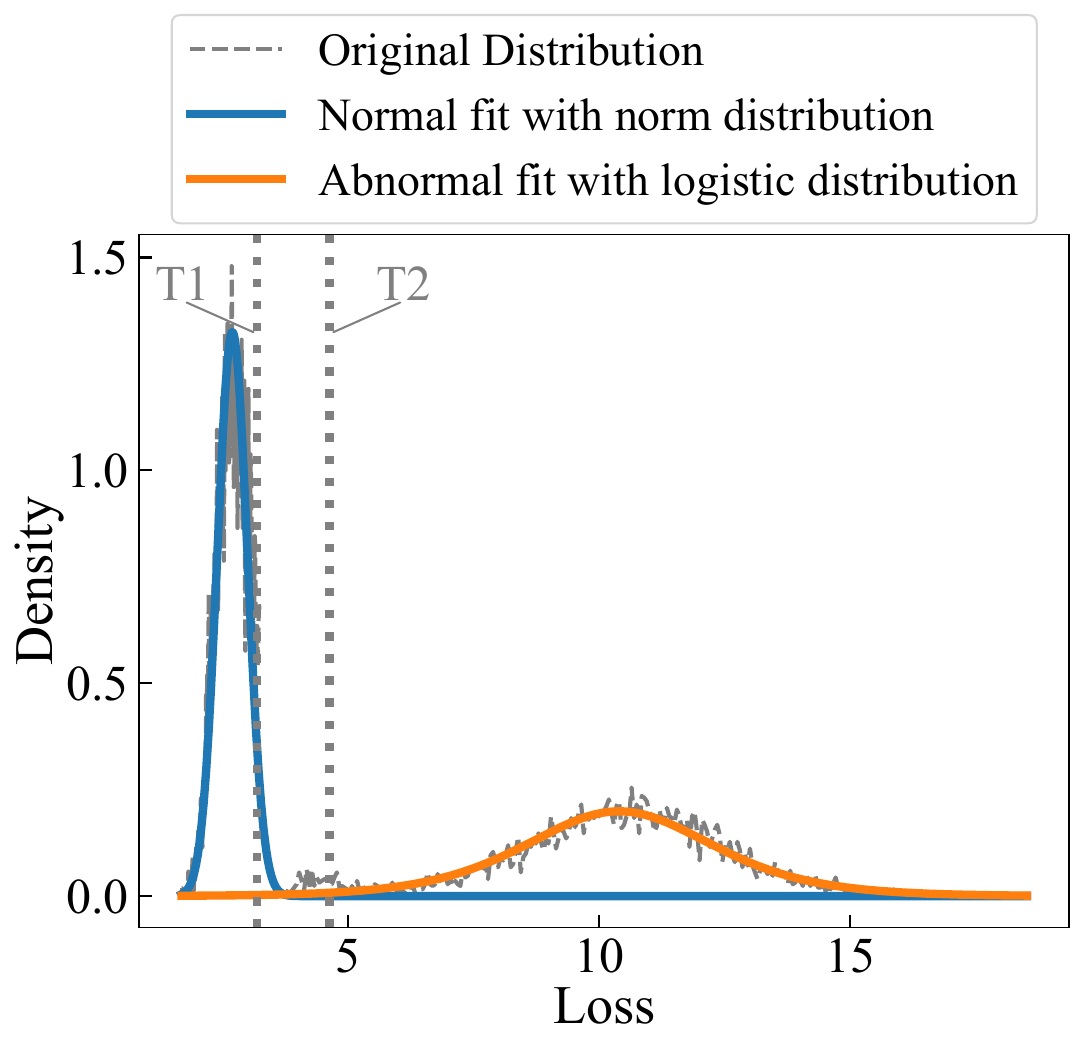}} \\
			\vspace{-1mm}
			(a) Network Distributions.
		\end{minipage}%
		\begin{minipage}{.24\textwidth}
			\centering
			\scalebox{.21}{\includegraphics{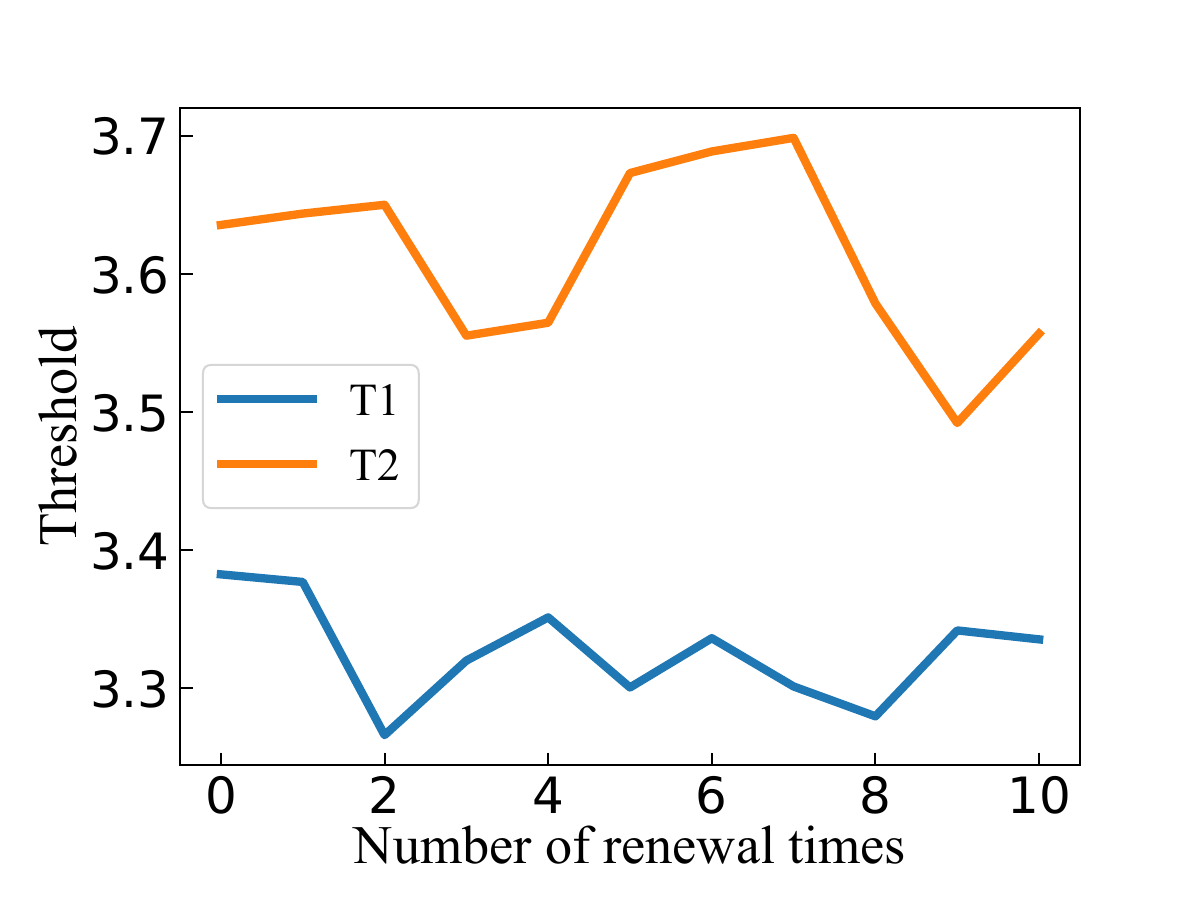}} 
			\vspace{-8mm}
			(b) Threshold update.
		\end{minipage}
		\caption{Adaptive threshold calculation on CIC-DoHBrw-2020.}
		\label{fi:networkN02DoHBrw}
		\vspace{-2mm}
	\end{figure}
	
	Fig. \ref{fi:lossDistrubutinoDoHBrw} illustrates the empirical distribution of the losses obtained from the Adaptive NAD model against the best-fitting theoretical distributions on the CIC-DoHBrw-2020 dataset. The description of renewed thresholds of the proposed adaptive NAD model for both $T_1$ and $T_2$ on the CIC-DoHBrw-2020 dataset is given in Fig. \ref{fi:networkN02DoHBrw}. 
	
\end{sloppypar}
\end{document}